\documentclass[sigconf]{acmart}

\usepackage{microtype}
\usepackage{graphicx}
\usepackage{color,soul}
\usepackage{booktabs} % for professional tables
\usepackage{subfig}
\usepackage{algorithm}
\usepackage{hyperref}
\usepackage{algpseudocode}%

\usepackage{fancyvrb}
\usepackage{fvextra}
\usepackage{listings}

\usepackage{pgfplots}
\usepackage{tikz-3dplot}

\tdplotsetmaincoords{60}{115}
\pgfplotsset{compat=newest}
\usepackage{amsmath}
\usepackage{mathtools}
\usepackage{amsthm}
\usepackage[capitalize,noabbrev]{cleveref}
\usepackage{xcolor}
\usepackage{titletoc}
\usepackage{lipsum}
\usepackage{caption}
\usepackage{subcaption}

\theoremstyle{plain}
\newtheorem{theorem}{Theorem}[section]

\newtheorem{lemma}[theorem]{Lemma}

\theoremstyle{definition}
\newtheorem{definition}[theorem]{Definition}

\theoremstyle{remark}

\AtBeginDocument{%
  }

\setcopyright{acmlicensed}
\copyrightyear{2025}
\acmYear{2025}
\acmDOI{XXXXXXX.XXXXXXX}

\acmConference[ICCPS '25]{ACM Conference}{June 03--05,
  2025}{Woodstock, NY}
\acmISBN{978-1-4503-XXXX-X/18/06}

\begin{document}

\title{Physics-model-guided Worst-case Sampling for Safe Reinforcement Learning}


\author{Hongpeng Cao}
\affiliation{
  \institution{Technical University of Munich}
  \city{Munich}
  \country{Germany}}
\email{cao.hongpeng@tum.de}

\author{Yanbing Mao}
\affiliation{%
  \institution{Wayne State University}
  \city{Detroit}
  \country{USA}
}
\email{hm9062@wayne.edu}

\author{Lui Sha}
\affiliation{%
 \institution{University of Illinois Urbana-Champaign}
 \city{Urbana}
 \country{USA}}
 \email{lrs@illinois.edu}

\author{Marco Caccamo}
\affiliation{
  \institution{Technical University of Munich}
  \city{Munich}
  \country{Germany}}
\email{mcaccamo@tum.de}

\renewcommand{\shortauthors}{xxx et al.}

\begin{abstract}
Real-world accidents in learning-enabled CPS frequently occur in challenging corner cases. During the training of deep reinforcement learning (DRL) policy, the standard setup for training conditions is either fixed at a single initial condition or uniformly sampled from the admissible state space. This setup often overlooks the challenging but safety-critical corner cases. To bridge this gap, this paper proposes a physics-model-guided worst-case sampling strategy for training safe policies that can handle safety-critical cases toward guaranteed safety. Furthermore, we integrate the proposed worst-case sampling strategy into the physics-regulated deep reinforcement learning (Phy-DRL) framework to build a more data-efficient and safe learning algorithm for safety-critical CPS. We validate the proposed training strategy with Phy-DRL through extensive experiments on a simulated cart-pole system, a 2D quadrotor, a simulated and a real quadruped robot, showing remarkably improved sampling efficiency to learn more robust safe policies. 
\end{abstract}




% \received{xx February xx}
% \received[revised]{xx March xx}
% \received[accepted]{xx June xx}

\maketitle

\section{Introduction}
Deep reinforcement learning (DRL) has been integrated into many cyber-physical systems (CPS; see examples in \cref{frame}), defining learning-enabled CPS that have succeeded tremendously in many complex control tasks. Notable examples range from autonomous driving \cite{kendall2019learning,kiran2021deep} to chemical processes \cite{savage2021model,he2021deep} to robot locomotion \cite{ibarz2021train,levine2016end}. Learning-enabled CPS promise to revolutionize many processes in different industries with tangible economic impact \cite{market1,market2}. However, the public-facing AI incident database \cite{AID} reveals that machine learning (ML) techniques, including DRL, can deliver much high performance but no safety assurance \cite{brief2021ai}.  For instance, in 2023, the US NHTSA reported nearly 224 crashes linked to self-driving and driver-assist technologies within a 9-month period~\cite{NHTSA}. Hence, a high-performance DRL with enhanced safety assurance is even more vital today, aligning well with the market’s need for ML safety.

\subsection{Related Work on Safe DRL}
To train a safe DRL policy, many literature adopts a constrained Markov decision process (CMDP) formulation, aiming to find a policy that jointly optimizes the objective of increasing the accumulated reward and decreasing the cost of safety violation \cite{achiam2017constrained, li2021augmented, wachi2020safe}. Furthermore, incorporating safety knowledge into the reward function design incentivizes the DRL to learn a safe policy. For instance, the control Lyapunov function (CLF) is widely used in constructing the safety-embedded reward \cite{perkins2002lyapunov, berkenkamp2017safe,chang2021stabilizing, zhao2023stable}. However, the safety of those learned policies can not be formally guaranteed due to the neural network parameterized policy, whose behaviors are hard to predict~\cite{huang2017adversarial} and verify~\cite{katz2017reluplex}.

Instead of focusing on learning a safe policy, the system-level safety framework sandboxes the unverified potential unsafe DRL policies regardless of the concrete design of the learning algorithm, and the safety is assured by an external verified safety controller \cite{{cai2024simplex, zhong2023towards, xiang2018verification, Humphrey2016Synthesis, Claviere2021Safety}}. However, those frameworks are often sensitive to the changes of the assumed dynamics models during the deployment.

Moreover, another focus of safe DRL has been shifted to integrating data-driven DRL action policy and physics-model-based action policy, leading to a residual action policy diagram \cite{rana2021bayesian,li2022equipping,cheng2019control,johannink2019residual, cheng2019end}. However, the physics models considered in those works are nonlinear and intractable, which thwarts delivering a verifiable safety, if not impossible. Recently, a physics-regulated deep reinforcement learning (Phy-DRL) framework \cite{Phydrl1, Phydrl2} is proposed to offer a promising solution to these challenges. Phy-DRL allows for the simplification of nonlinear system dynamics models into analyzable and verifiable linear models, which delivers a mathematically provable safety guarantee. Moreover, these linear models can then guide the construction of safety-embedded rewards and residual action policies. 

Nevertheless, in the aforementioned safe DRL frameworks, the policy learning setup is often fixed at a single initial condition or uniformly sampled from the admissible state space \cite{rana2021bayesian,li2022equipping,cheng2019control,johannink2019residual,cheng2019end,sasso2023posterior, Phydrl1, Phydrl2, muratore2022robot, liu2022goal, kirk2023survey}, which overlooks the challenging but safety-critical corner cases, potentially leading an unsafe policy.

\subsection{Open Problems}
In particular, incidents of learning-enabled CPS (e.g., self-driving cars) often occur in infrequent corner cases \cite{Phydrl1cos, Waymo00,bogdoll2021description}. This underscores that ``corner cases” induce a formidable safety challenge for DRL and other ML techniques. From a control-theoretic perspective, system-state samples close to the safety boundary represent the corner cases where a slight disturbance or fault can take a system out of control. Intuitively, focusing the training on such corner-case samples will enable a more robust and safe action policy. In the safe DRL community, how to define those corner cases and how to use them for learning safe policies remains unclear. 

\subsection{Core Contributions} \label{justcore}
To bridge the gap of training on corner cases in the existing literature, we propose a formal definition of the worst case for DRL based on the system's dynamics model. Furthermore, we propose an algorithm to efficiently generate the worst cases for policy learning. At last, we integrate the worst-case sampling into the Phy-DRL framework to learn safer and more robust policies. As shown in \cref{frame}, the integrated Phy-DRL framework defines worst-case samples as the state of the system located on the boundary of a safety envelope. These corner-case samples are not often visited during training via random sampling. Worst-case sampling thus lets Phy-DRL's training focus on the safety boundary, enabling a more robust and safe action policy.
We demonstrate the worst-case empowered Phy-DRL in three case studies including a cart-pole system, a 2D quadrotor, and a quadruped robot, showing remarkable improvement in sampling efficiency and safety assurance.

\begin{figure}
  \begin{center}
  \includegraphics[width=0.47\textwidth]{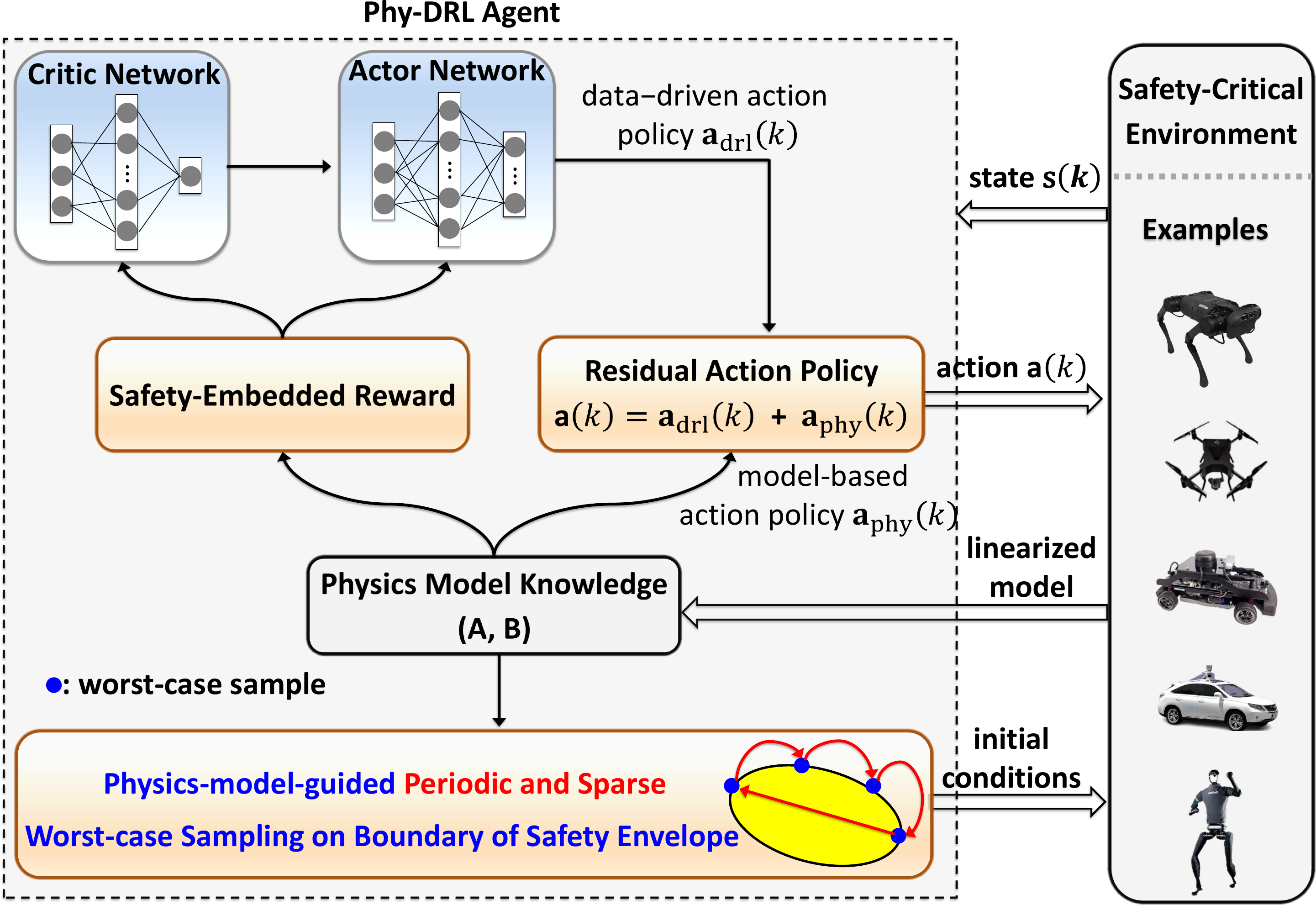}
  \end{center}
  \caption{Phy-DRL training powered by periodic and sparse worst-case sampling for safety-critical CPS.}
  \label{frame}
\end{figure}

\begin{table}[h]
\centering
\caption{Notations throughout Paper}
\begin{tabular}{|l|l|}
\hline
$\mathbb{R}^{n}$  &   set of $\emph{n}$-dimensional real vectors        \\ \hline
$\mathbb{N}$   &  set of natural numbers       \\ \hline
$[\mathbf{x}]_{i}$ & $i$-th entry of vector $\mathbf{x}$       \\ \hline
$[\mathbf{W}]_{i,:}$ & $i$-th row of matrix $\mathbf{W}$ \\ \hline$[\mathbf{W}]_{i,j}$ & matrix $\mathbf{W}$'s element at row $i$ and column $j$\\ \hline
$\mathbf{P} \succ \!(\prec)~0$ & matrix $\mathbf{P}$ is positive (negative) definite \\ \hline  
$\top$ & matrix or vector transposition  \\ \hline
$\left|  \cdot  \right|$ & set cardinality,  or absolute value \\ \hline
$\mathbf{I}_{n}$  &   $\emph{n}$-dimensional identity matrix       \\ \hline
$\mathbb{X} \setminus \Omega$ & complement set of $\Omega$  with respect to $\mathbb{X}$\\ \hline
\end{tabular} \label{notation}
\end{table}

\section{Preliminaries}
% \underline{Note: \cref{notation} in \cref{appnotation} summarizes notations used throughout the paper}. 
\subsection{Notations} \label{appnotation}
We summarize notations used through the paper in \cref{notation}

\subsection{Safety Definition}
The dynamics model of a real plant can be described by 
\begin{align}
\mathbf{s}(k+1) = {\mathbf{A}} \cdot \mathbf{s}(k) + {\mathbf{B}} \cdot \mathbf{a}(k) + \mathbf{f}(\mathbf{s}(k), \mathbf{a}(k)), ~k \in \mathbb{N}  \label{realsys}
\end{align}
where $\mathbf{f}(\mathbf{s}(k), \mathbf{a}(k)) \in \mathbb{R}^{n}$ is the unknown model mismatch, ${\mathbf{A}} \in \mathbb{R}^{n \times n}$ and ${\mathbf{B}} \in \mathbb{R}^{n \times m}$ denote known
system matrix and control structure matrix, respectively,  $\mathbf{s}(k) \in \mathbb{R}^{n}$ is system state, $\mathbf{a}(k) \in \mathbb{R}^{m}$ is action. The available knowledge of the model related to the real plant (\ref{realsys}) is represented by $\left({\mathbf{A},~ \mathbf{B}} \right)$. We are interested in an action policy that can constrain the system states to the safety set $\mathbb{X}$:
\begin{align}
{\mathbb{X}} \!\triangleq\! \left\{ {\left. \mathbf{s} \!\in\! {\mathbb{R}^n} \right|\underline{\mathbf{v}} \le {\mathbf{D}} \cdot \mathbf{s} - \mathbf{v} \le \overline{\mathbf{v}}}, \!~\mathbf{D} \in \mathbb{R}^{h \times n}, \!\!\right. \left. \text{with}~\mathbf{v}, \overline{\mathbf{v}}, \underline{\mathbf{v}} \in \mathbb{R}^{h}  \right\}, \label{aset2}
\end{align}
where $\mathbf{D}$, $\mathbf{v}$, $\overline{\mathbf{v}}$ and $\underline{\mathbf{v}}$ are given in advance for formulating $h \in \mathbb{N}$ safety conditions. To guarantee the system always stays in the safety set, we introduce a subset of the safety set $\mathbb{X}$ called safety envelope $\Omega$ based on \textit{Lyapunov-stability theorem}. 
\begin{align}
\text{Safety Envelope}~{\Omega} \triangleq \left\{ {\left. {\mathbf{s} \in {\mathbb{R}^n}} \right|{\mathbf{s}^\top}\cdot{\mathbf{P}}\cdot\mathbf{s} \le 1,~{\mathbf{P}} \succ 0} \right\}, \label{set3}
\end{align}
where $\mathbf{P}\in \mathbb{R}^{n \times n}$ is a positive definite matrix, that defines the shape of $\Omega$. With the safety envelope $\Omega$, the safety problem is defined as the follows:

\begin{definition}  \cite{Phydrl1}
Consider the safety set $\mathbb{X}$ \eqref{aset2} and the safety envelop $\Omega$ \eqref{set3}. The real plant \eqref{realsys} is said to be safe, if given any $\mathbf{s}(1) \in \Omega \subseteq   {\mathbb{X}}$, the $\mathbf{s}(k) \in \Omega \subseteq {\mathbb{X}}$ holds for any time $k \in \mathbb{N}$.
\label{defsafety}
\end{definition}

To guarantee a system controlled by a DRL agent staying in the safety envelope $\Omega$ is non-trivial due to its unverifiable action output. The recent literature Phy-DRL \cite{Phydrl1} suggests that incorporating the knowledge of the physics model into the standard DRL framework can significantly improve safety assurance toward guaranteed safety. We summarize the Phy-DRL framework in the next section.

\subsection{Phy-DRL Agent}
Phy-DRL is built on the deterministic policy algorithms \cite{lillicrap2015continuous, fujimoto2018addressing}. As shown in \cref{frame}, its control action is in residual form: 
\begin{align}
\mathbf{a}(k) = \underbrace{\mathbf{a}_{\text{drl}}(k)}_{\text{data-driven}} + \underbrace{\mathbf{a}_{\text{phy}}(k) ~(:=  \mathbf{F} \cdot \mathbf{s}(k))}_{\text{model-based}},\label{residual}
\end{align}
where $\mathbf{a}_{\text{drl}}(k)$ denotes a date-driven action from DRL, while $\mathbf{a}_{\text{phy}}(k)$ is a model-based action. Meanwhile, Phy-DRL embeds safety envelope \eqref{set3} into reward design, creating safety-embedded reward: 
\begin{align}\label{reward}
&\mathcal{R}( {\mathbf{s}(k),\mathbf{a}_{\text{drl}}}(k)) \\ \nonumber
&= \underbrace{\mathbf{s}^\top(k) \cdot  \mathbf{H}  \cdot \mathbf{s}(k)  - {\mathbf{s}^\top(k+1)} \cdot \mathbf{P} \cdot \mathbf{s}(k+1)}_{\triangleq ~ c(\mathbf{s}(k),~\mathbf{s}(k+1))}   ~+~ w( \mathbf{s}(k),\mathbf{a}(k)), 
\end{align}
where the term $w( \mathbf{s}(k),\mathbf{a}(k))$ aims at high-performance operations (e.g., minimizing energy consumption of resource-limited robots \cite{yang2021learning,gangapurwala2020guided}). The term $c(\mathbf{s}(k),\mathbf{s}(k+1))$ is safety-critical, in which 
\begin{equation}\label{hmatrix}
\mathbf{H} \triangleq  {{\overline{\mathbf{A}}^\top} \cdot \mathbf{P} \cdot \overline{\mathbf{A}}}, 
~\text{with}~\overline{\mathbf{A}} \buildrel \Delta \over = \mathbf{A} + {\mathbf{B} } \cdot {\mathbf{F}}~\text{and}~0 ~\prec~  \mathbf{H} ~\prec~ \alpha \cdot \mathbf{P},~~\alpha \in (0,1),
\end{equation}
where $\alpha$ is a pre-defined parameter to determine the decrease rate of the Lyapunov value. The matrix $\mathbf{P}$ is the matrix for building the safety envelope $\Omega$ \eqref{set3} and $\mathbf{F}$ is the feedback control law. With the available physics-model knowledge $(\mathbf{A}, \mathbf{B})$ at hand, the $\mathbf{F}$ and $\mathbf{P}$ can be computed using LMI toolbox \cite{gahinet1994lmi, boyd1994linear}. We refer interested readers to \cite{Phydrl1} for a more detailed explanation of LMI formulations. The intuition of the sub-reward $c(\mathbf{s}(k),\mathbf{s}(k+1))$ is that we encourage the DRL agent to learn a safe policy in conjunction with the model-based policy $a_{phy}$ to stabilize the real plant \eqref{realsys} toward the equilibrium point.

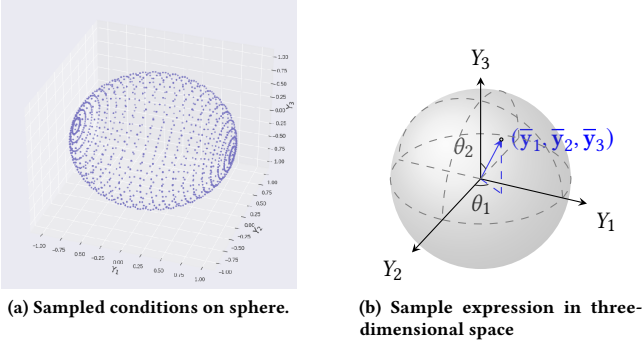
\begin{figure}[h]
    \subfloat[Sampled conditions on sphere.]{\includegraphics[width=0.22\textwidth]{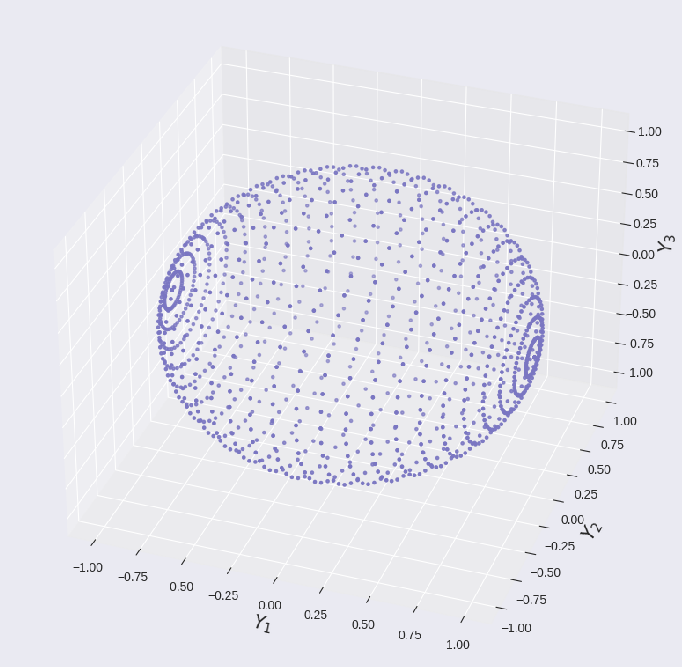}}\hfill
    \subfloat[Sample expression in three-dimensional space]{
        \begin{minipage}[t]{0.2\textwidth}
            \centering
            \begin{tikzpicture}[tdplot_main_coords, scale=1.2]
            % Create a point (P)
            \coordinate (O) at (0,0,0);
            \pgfmathsetmacro{\rvec}{.8}
            \pgfmathsetmacro{\thetavec}{30}
            \pgfmathsetmacro{\phivec}{60}
            \tdplotsetcoord{P}{\rvec}{\thetavec}{\phivec}
            \draw[-stealth,color=blue] (O) -- (P) node[right] {$(\overline{\mathbf{y}}_{1}, \overline{\mathbf{y}}_{2}, \overline{\mathbf{y}}_{3})$};
            \draw[dashed, color=blue] (O) -- (Pxy);
            \draw[dashed, color=blue] (P) -- (Pxy);
            \tdplotdrawarc{(O)}{0.15}{0}{\phivec}{anchor=north}{${\theta_1}$}
            \tdplotsetthetaplanecoords{\phivec}
            \tdplotdrawarc[tdplot_rotated_coords]{(0,0,0)}{0.2}{0}{\thetavec}{anchor=south east}{${\theta_2}$}
            % Draw shaded circle
            \shade[ball color = lightgray, opacity = 0.3] (0,0,0) circle (1cm);
            % draw arcs 
            \tdplotsetrotatedcoords{0}{0}{0};
            \draw[dashed, tdplot_rotated_coords, gray] (0,0,0) circle (1);
            \tdplotsetrotatedcoords{90}{90}{90};
            \draw[dashed, tdplot_rotated_coords, gray] (1,0,0) arc (0:180:1);
            \tdplotsetrotatedcoords{0}{90}{90};
            \draw[dashed, tdplot_rotated_coords, gray] (1,0,0) arc (0:180:1);
            % Axes in 3 d coordinate system
            \draw[-stealth] (0,0,0) -- (1.80,0,0) node[below left] {$Y_2$};
            \draw[-stealth] (0,0,0) -- (0,1.30,0) node[below right] {$Y_1$};
            \draw[-stealth] (0,0,0) -- (0,0,1.30) node[above] {$Y_3$};
            \draw[dashed, gray] (0,0,0) -- (-1,0,0);
            \draw[dashed, gray] (0,0,0) -- (0,-1,0);
            \draw[fill = lightgray!50] (P) circle (0.5pt);
            \end{tikzpicture}
        \end{minipage} 
    }
    \caption{Worst-case condition generation in for a three-dimensional $(n=3)$ safety envelope.}
    \label{fig: visualization}
\end{figure}

\section{Worst-case Sampling for Phy-DRL Training} 
\label{sec3}
The safety envelope is centered at the control equilibrium point. The plant is more likely to violate the safety constraint when its state is near the envelope boundary. For a DRL-controlled system, it is hard to certify and predict the output of the DRL output due to its non-convexity and non-linearity. Therefore, ensuring the safety of the DRL at the boundary of the envelope is critical. In this section, we propose a definition of the worst conditions for DRL in a safety-critical system and a practical algorithm to generate these conditions for DRL training and testing. 

%\begin{figure}
%  \begin{center}
%\includegraphics[width=0.3\textwidth]{figures/one.pdf}
%  \end{center}
%  \caption{The plot illustrates one example two dimensional safety envelope $\Omega \in \mathbb{R}^2$ with one step safe region denoted as $\Delta_{\text{one}}$ and the potential unsafe region indicated as $\Omega \setminus  \Delta_{\text{one}}$. The plant is more likely to violate the safety constraints when its state is near the boundary of $\Omega$.}
%  \label{expdrfg}
%\end{figure}

\begin{definition}[Worst-case conditions]
Referring to the safety envelop in \cref{set3}, a state $\mathbf{s}$ is said to be a worst-case condition if ${\mathbf{s}^\top}\cdot{\mathbf{P}}\cdot\mathbf{s} = 1$ (i.e., locating on the boundary of safety envelope). 
\label{worstcasesampledef}
\end{definition} 

Recall that, given $n$ dimensional safety constraints, the solution for the safety envelope becomes a $n$ dimensional hyperellipsoid, to generate worst-case samples referring to \cref{worstcasesampledef}, we need to solve ${\mathbf{s}^\top}\cdot{\mathbf{P}}\cdot\mathbf{s} = 1$,  where $\mathbf{s} \in\mathbb{R}^n$. To achieve this, we present the following lemma for having explicit solutions of worst-case conditions.

\begin{lemma}
Given $\mathbf{P} \succ 0$, the solution of $\mathbf{s} \in \mathbb{R}^{n}$, being subject to ${\mathbf{s}^\top}\cdot {\mathbf{P}} \cdot \mathbf{s} = \varphi$, is 
\begin{align}
\mathbf{s} = \mathbf{Q}(\mathbf{P}) \cdot \mathbf{y}, ~~~~\text{with}~~[{\mathbf{y}}]_{i} = \begin{cases}
		\!\sqrt{\frac{\varphi}{{{\lambda _1}( \mathbf{P})}}} \cdot \sin ({{\theta_1}}) \cdot \prod\limits_{m = 2}^{n - 1} {\sin } ({{\theta _m}}), \!\!&i = 1 \\
        \!\sqrt{\frac{\varphi}{{{\lambda_i}( \mathbf{P})}}} \cdot \cos ({{\theta_{i-1}}}) \cdot \prod\limits_{m = i}^{n - 1} {\sin } ({{\theta _m}}), \!\!&i \ge 2 
	\end{cases} \label{solutionsboundary2}
\end{align}
where $\mathbf{Q}(\mathbf{P})$ is $\mathbf{P}$'s orthogonal matrix, and ${\lambda_i}( \mathbf{P})$ is the $i$-th eigenvalue of matrix $\mathbf{P} \in \mathbb{R}^{n \times n}$. 
\label{solutionsboundary}
\end{lemma}

\begin{proof}
See \cref{pfworst}.
\end{proof}

\textbf{Example in 3D case}: Consider a three-dimensional safety envelop, i.e., $\mathbf{s} \in \mathbb{R}^{3}$. In this example, according to \cref{pso3}, we have 
\begin{align}
\sum\limits_{i = 1}^3 {\frac{{{\lambda _i}(\mathbf{P})}}{\varphi}}  \cdot \mathbf{y}_i^2 = 1, ~~~~\mathbf{y} = \mathbf{Q}^{\top}(\mathbf{P}) \cdot \mathbf{s}
\label{eq: yi}
\end{align}
for which we define
\begin{align}
\overline{\mathbf{y}}_i \triangleq \sqrt{{\frac{{{\lambda _i}(\mathbf{P})}}{\varphi}}}  \cdot \mathbf{y}_i, 
\label{eq: defiY}
\end{align}
in light of \eqref{eq: defiY}, \cref{eq: yi} can be rewritten as 
\begin{align}
\overline{\mathbf{y}}_{1}^2 + \overline{\mathbf{y}}_{2}^2 + \overline{\mathbf{y}}_{3}^2 = 1, \label{eq: defiY111}
\end{align}
which describes a sphere in $\mathbb{R}^{3}$ space, as shown in~\cref{fig: visualization} (a). From~\cref{fig: visualization} (b), we notice that every point on the sphere can be parameterized using angles $\theta_1$ and $\theta_2$ as: 
\begin{align}
 \overline{\mathbf{y}}_{i} \triangleq \begin{cases}
	{\sin} ({{\theta_2}}) \cdot {\sin} ({{\theta_1}}), &i = 1 \\
        {\sin} ({{\theta_2}}) \cdot {\cos} ({{\theta_1}}) , &i = 2 \\
        {\cos} ({{\theta_2}}), &i = 3.
	\end{cases}
\label{eq: defiYoo}
\end{align}

By selecting different values for $\theta_1$ and $\theta_2$, we can sample any point $[\overline{\mathbf{y}}_{1}, \overline{\mathbf{y}}_{2}, \overline{\mathbf{y}}_{3}]^\top$ on the sphere and, consequently, obtain the conditions on the boundary of the safety envelop (indicated by \cref{eq: yi} and \cref{eq: defiY})

\begin{figure}
    \begin{minipage}{0.46\textwidth}
      \begin{algorithm}[H]
        \caption{\normalsize{Periodic and Worst-case Sampling for Phy-DRL Training}}
        \begin{algorithmic}[1]
        \State \textbf{Input:} System-state dimension $n \in \mathbb{N}$; sample numbers $\bar{q}_{r} \in \mathbb{N}$, $r = 1, \ldots, n-1$; matrix $\mathbf{P}$; parameter $\varphi = 1$, Period number $p \in \mathbb{N}$.
        \State Initialize boundary set: $\mathbb{B}^{\mathbf{P}}_\varphi$ $\leftarrow \emptyset$;

      \For{$\theta_1 = 0 : \frac{2\pi}{\bar{q}_{1}} : (2\pi-\frac{2\pi}{\bar{q}_{1}})$} \Comment{Generating worst-case conditions} \label{start_gen}
            \State Set: $\theta_{2} \leftarrow 0, \theta_{3} \leftarrow 0, \dots, \theta_{n-1} \leftarrow 0$; \label{assignzero}
            \State Generate $\mathbf{s} \in \mathbb{R}^{n}$ by  \cref{solutionsboundary2};
              \State Update set: $\mathbb{B}^{\mathbf{P}}_\varphi \leftarrow \mathbb{B}^{\mathbf{P}}_\varphi \cup \left\{ \mathbf{s} \right\}$;
               \For{$\theta_{2} = \frac{2\pi}{\bar{q}_{2}} : \frac{2\pi}{\bar{q}_{2}} : (2\pi-\frac{2\pi}{\bar{q}_{2}})$} \label{for7}
                \State $\textbf{\vdots}$  \vspace{0.20cm}   
                \For{$\theta_{n-1} = \frac{2\pi}{\bar{q}_{n-1}} : \frac{2\pi}{\bar{q}_{n-1}} : (2\pi-\frac{2\pi}{\bar{q}_{n-1}})$} \label{forn-1}
                \State Generate $\mathbf{s} \in \mathbb{R}^{n}$ by  \cref{solutionsboundary2};
               \State Update set: $\mathbb{B}^{\mathbf{P}}_\varphi \leftarrow \mathbb{B}^{\mathbf{P}}_\varphi \cup \left\{ \mathbf{s} \right\}$;
       \EndFor 
       \State  $\textbf{\vdots}$  
       \EndFor
       \EndFor \label{end_gen}
        \For{$j=1$ to $p$} \Comment{Start training curriculum} \label{start_cur}
        \For{$s\in \mathbb{B}^{\mathbf{P}}_\varphi$}
         \State Set $\mathbf{s}(1) \leftarrow \mathbf{s}$ for system in \cref{realsys};
          \State Train (test) Phy-DRL agent for one episode;
           \EndFor
           \EndFor \label{end_cur}
        \end{algorithmic}\label{ALG1}
      \end{algorithm}
    \end{minipage}
  \end{figure}

We now are ready to propose an algorithm to automatically generate sampling conditions located at the safety envelope's boundary. Moreover, we design a training curriculum to periodically visit the worst-case conditions for policy learning. 

As shown in \cref{ALG1}, the proposed algorithm includes worst-case condition generation \cref{start_gen}--\cref{end_gen}, and training curriculum \cref{start_cur} -- \cref{end_cur}. In worst-case condition generation, we sample $\theta_1, \theta_2, \ldots, \theta_{n-1}$ sparsely in the interval $[0,~2\pi)$. This is motivated by the solutions of worst-case samples in \cref{solutionsboundary2}. The number of samples $n$ is a user-defined parameter to determine the sparsity of the samples. As shown in \cref{exp}, we found that training on a few worst-case conditions can already learn a safe policy that renders the safety envelope invariant. In the training curriculum, we iteratively set the system's state at the generated worst-case conditions. This can help the agent to sufficiently visit the challenging conditions for learning a safe and robust policy.

\textbf{Episode Complexity.}
\cref{start_gen}, \cref{for7} and \cref{forn-1} indicate the sampling number of each radian $\theta_{r}$ is $q_{r} -1$, where $r = 2, \ldots, n-1$, while the number of $\theta_{1}$ is $q_{n}$. This setting is because of the observation from \cref{solutionsboundary2} that if  $\theta_{i}$ = 0, the values of radians $\theta_{1}, \ldots, \theta_{i-1}$ have no influence on solution $\mathbf{y}$. For example, if we let $\theta_{n-1} = 0$, even without knowing $\theta_{1}, \ldots, \theta_{n-2}$, the ${\mathbf{y}}$ can be directly obtained as $\mathbf{y} = \sqrt{\frac{\varphi}{{{\lambda _1}( \mathbf{P})}}}[0, 0, \ldots, 0, 1]^{\top}$. This observation motivates \cref{assignzero} of \cref{ALG1}. The total number of episodes (or sampled worst-case samples) is thus 
\begin{align}
\text{number of episodes} = {{q_1} \cdot p \cdot \prod\limits_{i = 2}^{n - 1} {\left( {{q_i} - 1} \right)}  + {q_1} \cdot p }. \label{numeps} 
\end{align}

\begin{figure*}[h]
    \centering
    \subfloat[worst-case]{\includegraphics[width=0.25\textwidth]{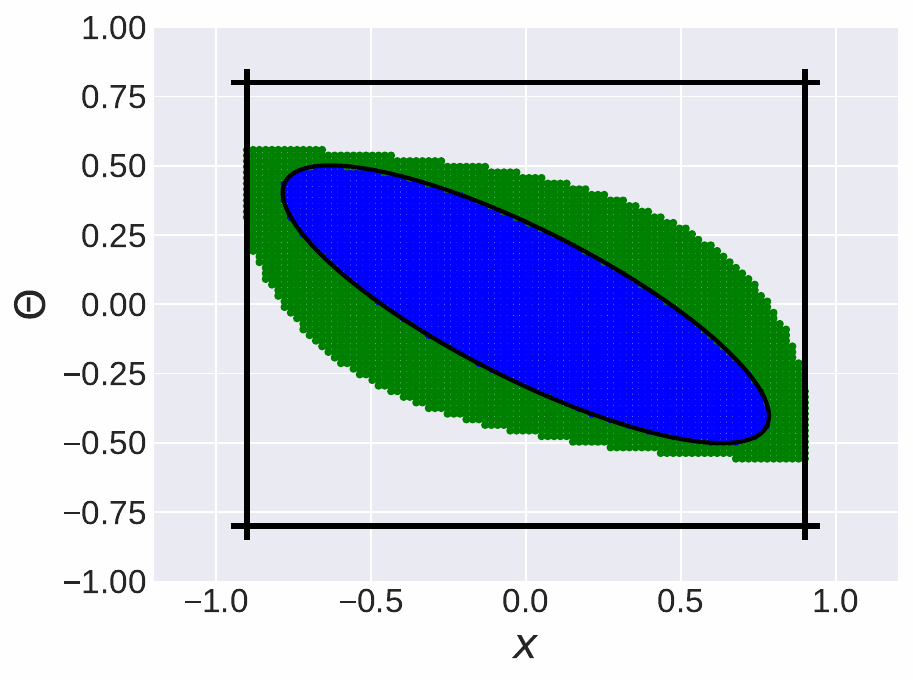}} 
    \centering
    \subfloat[random]{\includegraphics[width=0.25\textwidth]{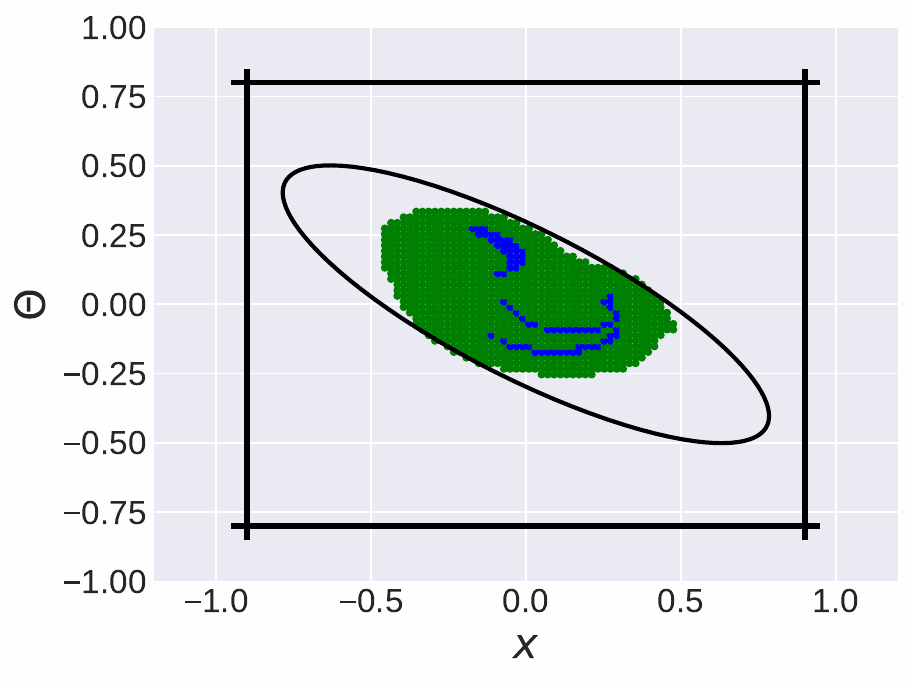}} 
    \centering
    \subfloat[worst-case]{\includegraphics[width=0.25\textwidth]{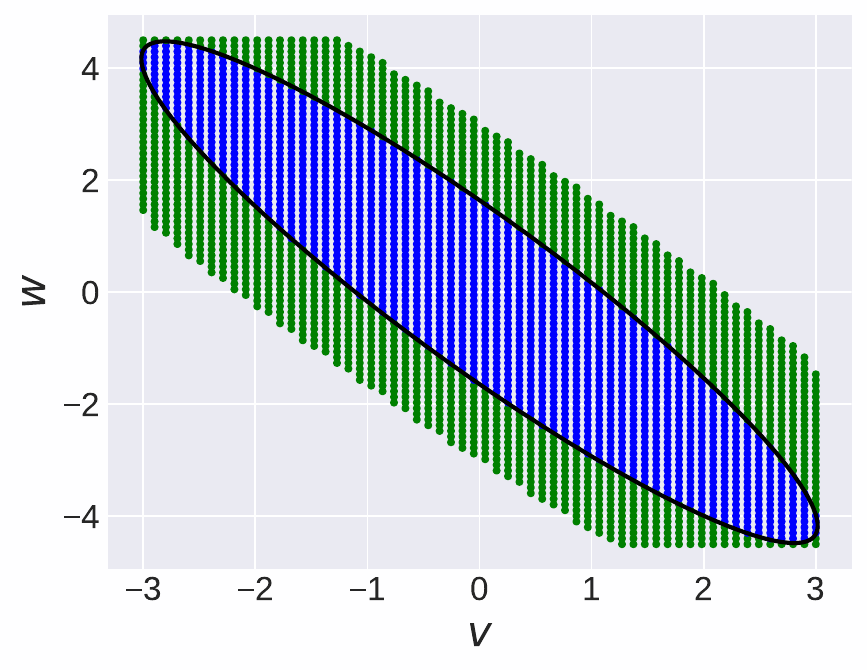}} 
    \centering
    \subfloat[random]{\includegraphics[width=0.25\textwidth]{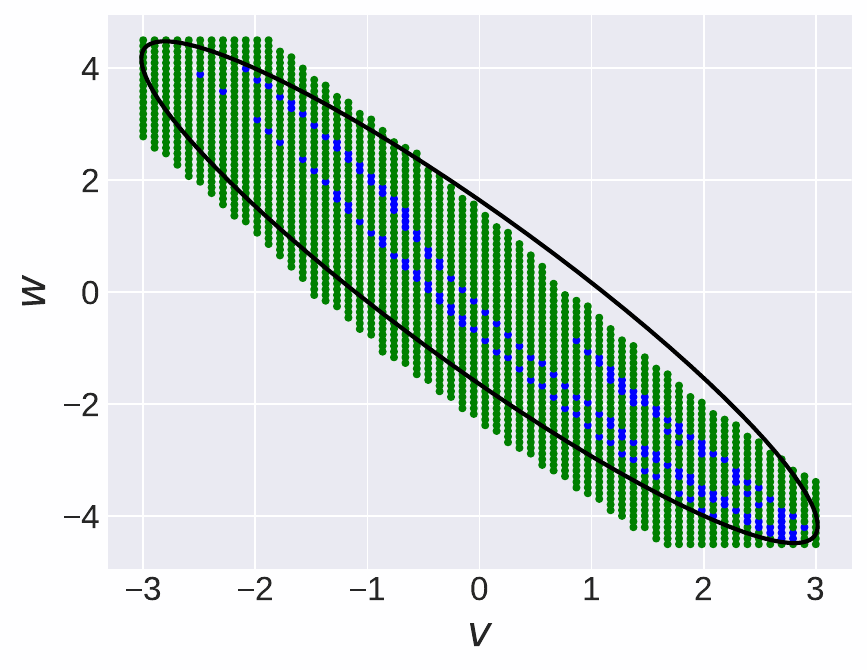}} 
\caption{Worst-case Sampling \textit{v.s.} Random Sampling, with termination condition. Blue: area of IE samples \eqref{ies}. Green: area of EE samples \eqref{ees}. Ellipse area: safety envelope. The (a) and (b) are the testing result visualized on $x$ and $\theta$ dimensions, where (c) and (d) are the results visualized on $v$ and $w$ dimensions. The size of colored area indicates the safety and robustness of the learned policy, the larger the better.}
\label{safesample_x_theta_w}
\end{figure*}

\begin{figure*} [http]
    \centering
    \subfloat[worst-case-w.t.]{\includegraphics[width=0.25\textwidth]{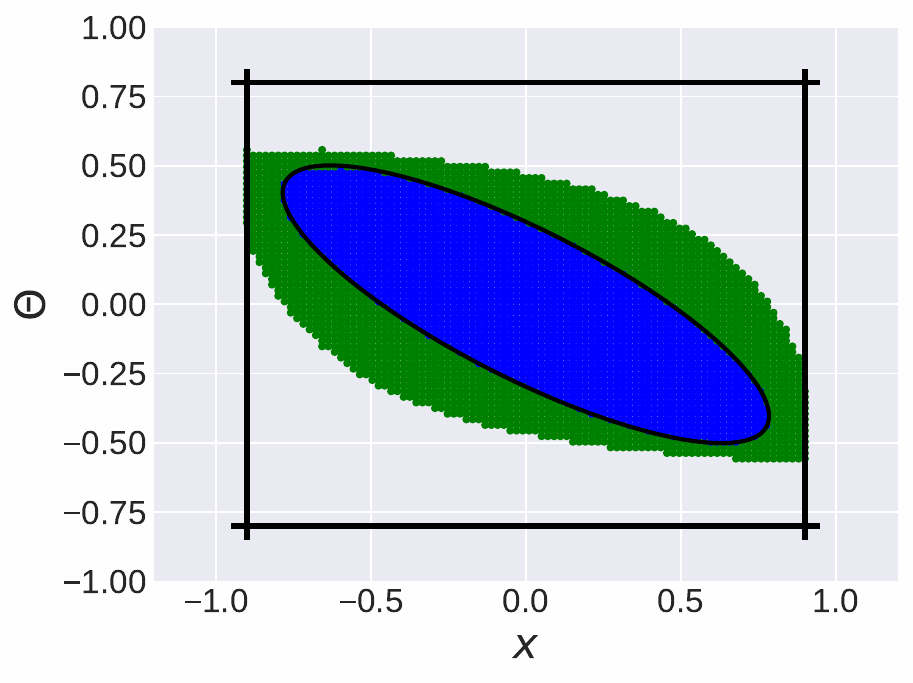}} 
    \centering
    \subfloat[random-w.t.]{\includegraphics[width=0.25\textwidth]{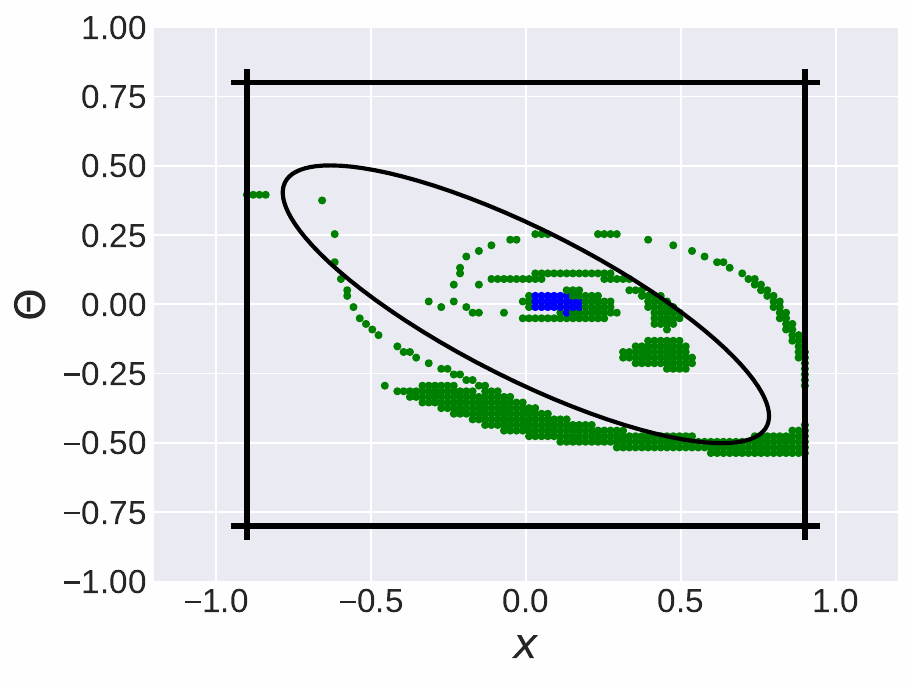}} 
    \centering
    \subfloat[worst-case-w.t.]{\includegraphics[width=0.25\textwidth]{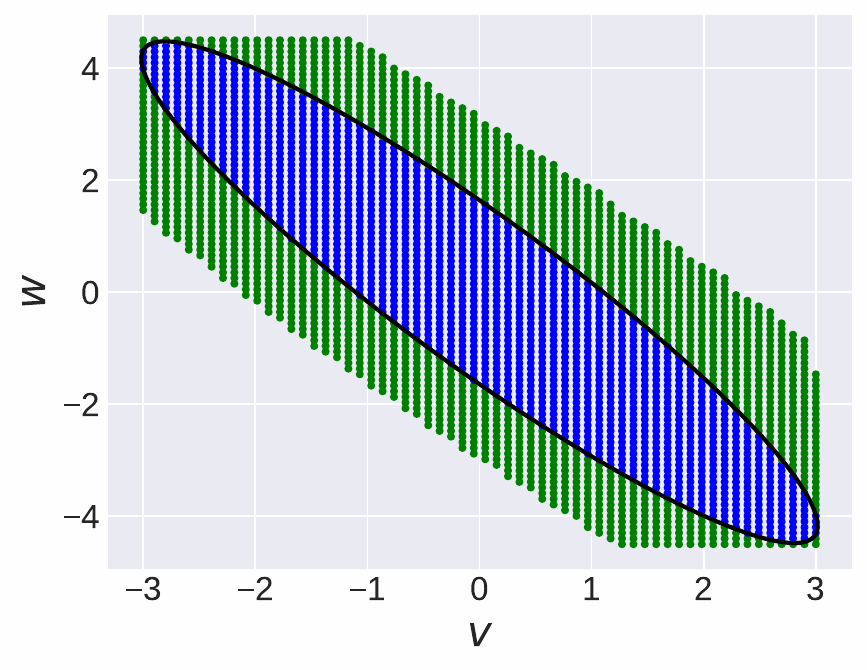}} 
    \centering
    \subfloat[random-w.t.]{\includegraphics[width=0.25\textwidth]{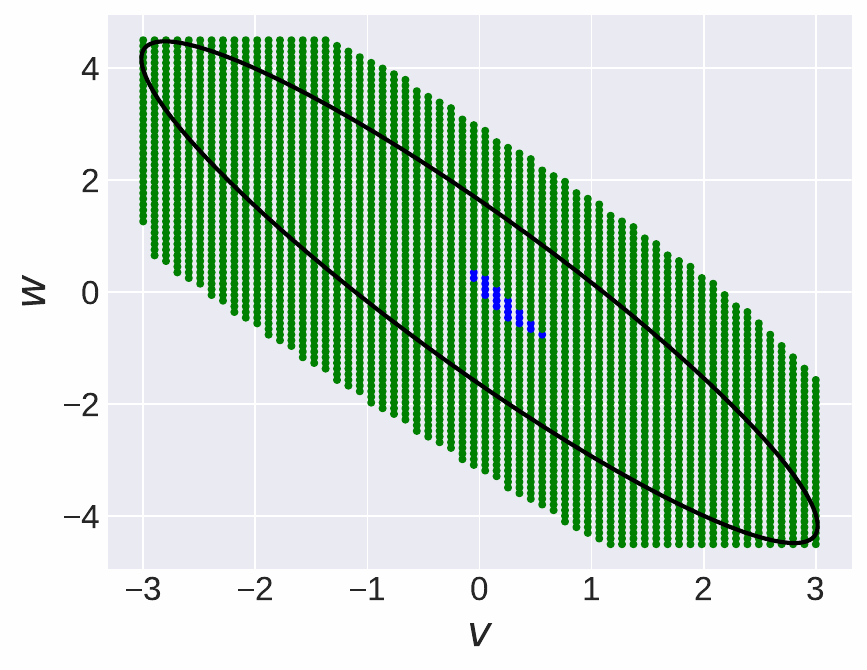}} 
\caption{Worst-case Sampling \textit{v.s.} Random Sampling, without using termination condition in training. }
\label{safesample_x_theta_wo}
\end{figure*}

\begin{figure*}[http]
    \centering
    \subfloat[$\text{Phy-DRL}_{\text{wc}}$]{\includegraphics[width=0.25\textwidth]{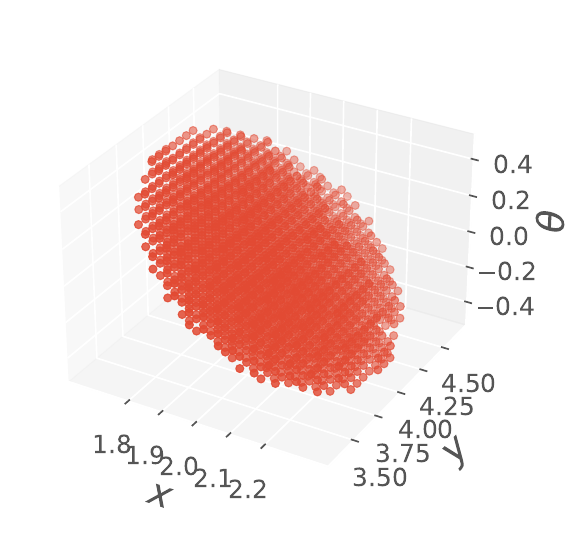}} 
    \centering
    \subfloat[$\text{Phy-DRL}_{\text{ran}}$]{\includegraphics[width=0.25\textwidth]{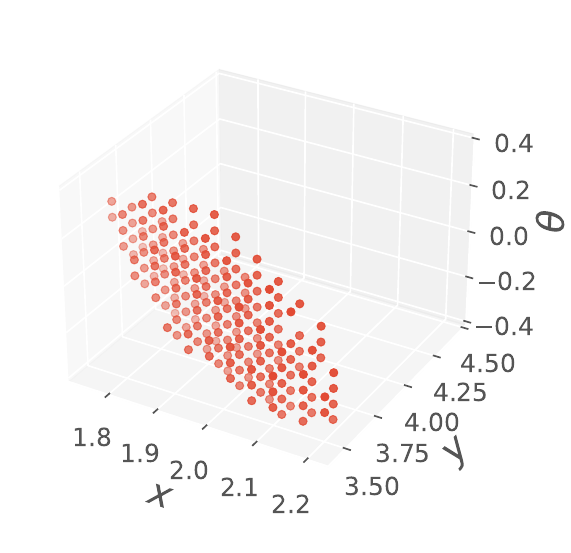}} 
    \centering
    \subfloat[$\text{DRL}_{\text{CLF-wc}}$]{\includegraphics[width=0.25\textwidth]{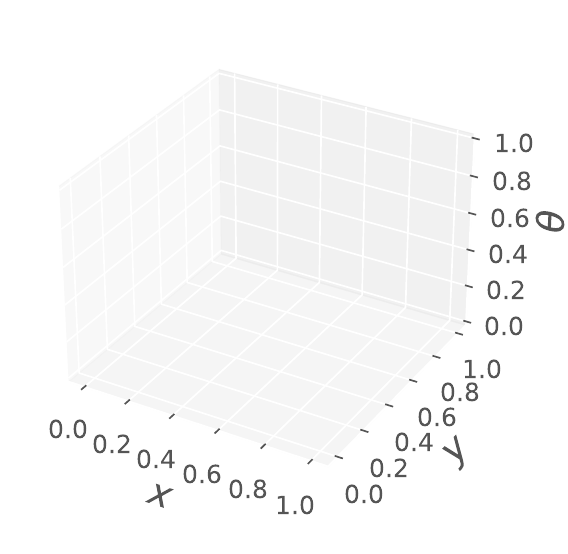}} 
    \centering
    \subfloat[$\text{average reward}$]{\includegraphics[width=0.25\textwidth]{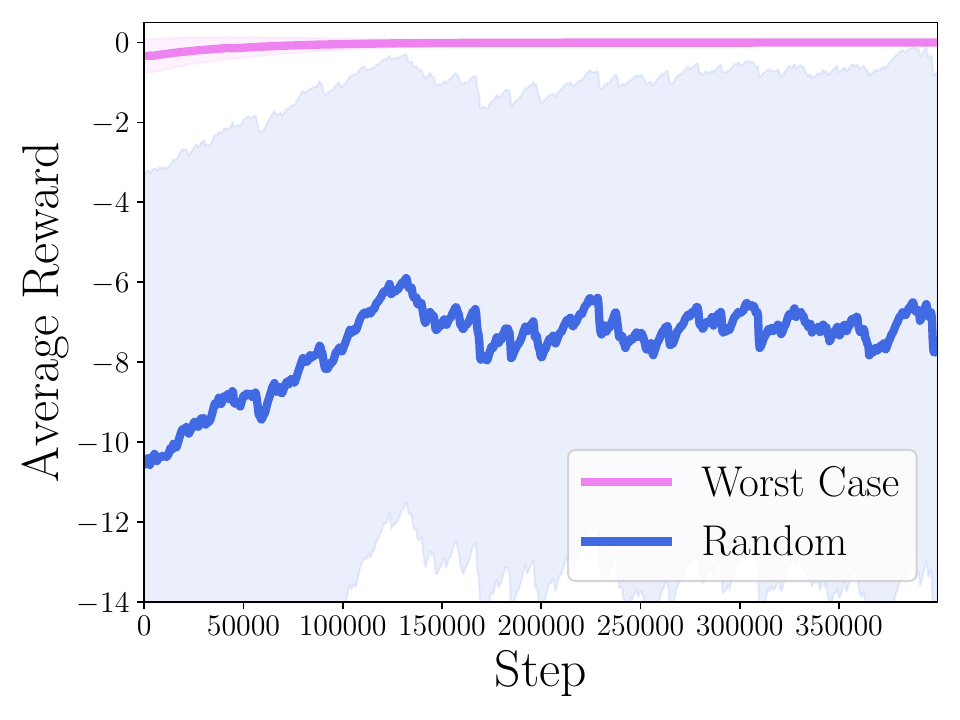}} 
\caption{(a)-(c): The number and locations of IE samples \eqref{ies} visualized in $x-y-\theta$ space. $\text{Phy-DRL}_{\text{wc}}$ has much more colored points, meaning that it can almost render the safety envelope invariant. (d): Reward curves (five random seeds): $\text{Phy-DRL}_{\text{wc}}$  v.s. $\text{Phy-DRL}_{\text{ran}}$.}
\label{quaapcv}
\end{figure*}

\begin{figure*}[h]
    \centering
    \subfloat[snow road: $r_{\text{x}}= 1$ m/s]{\includegraphics[width=0.245\textwidth]{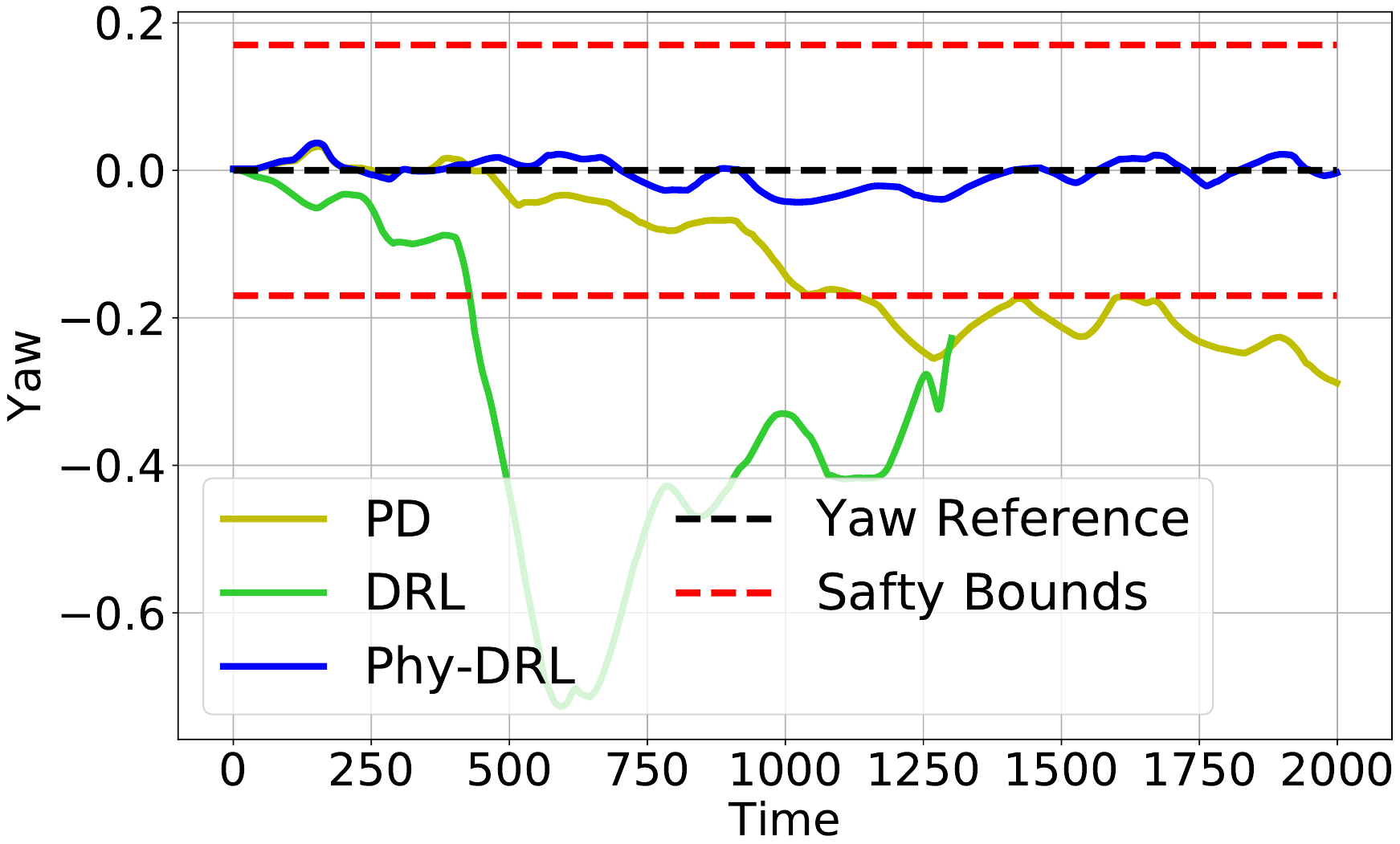}} 
    \centering
    \subfloat[snow road: $r_{\text{x}}= 1$ m/s]{\includegraphics[width=0.245\textwidth]{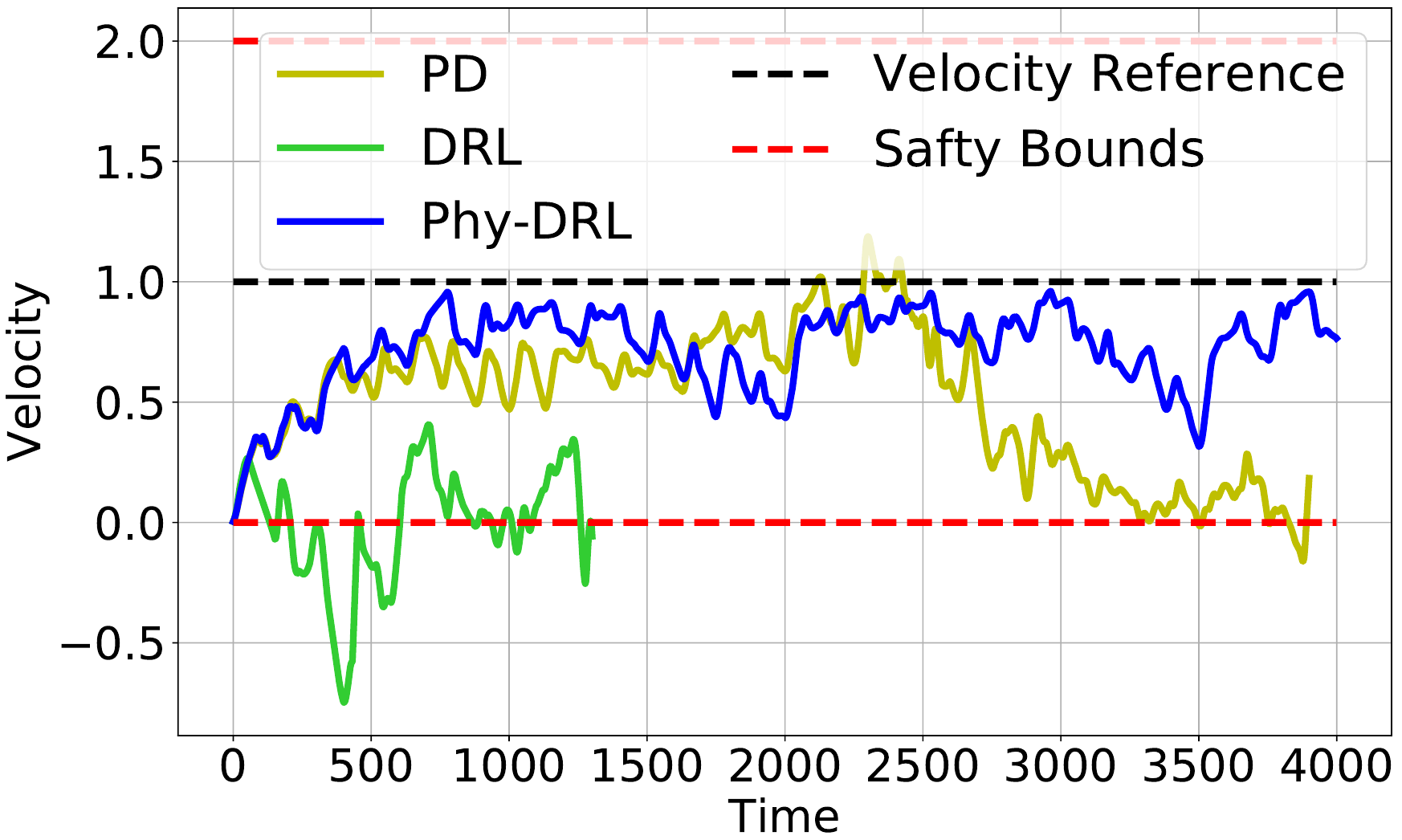}} 
    \centering
        \subfloat[wet road: $r_{\text{x}}=  -1$ m/s]{\includegraphics[width=0.245\textwidth]{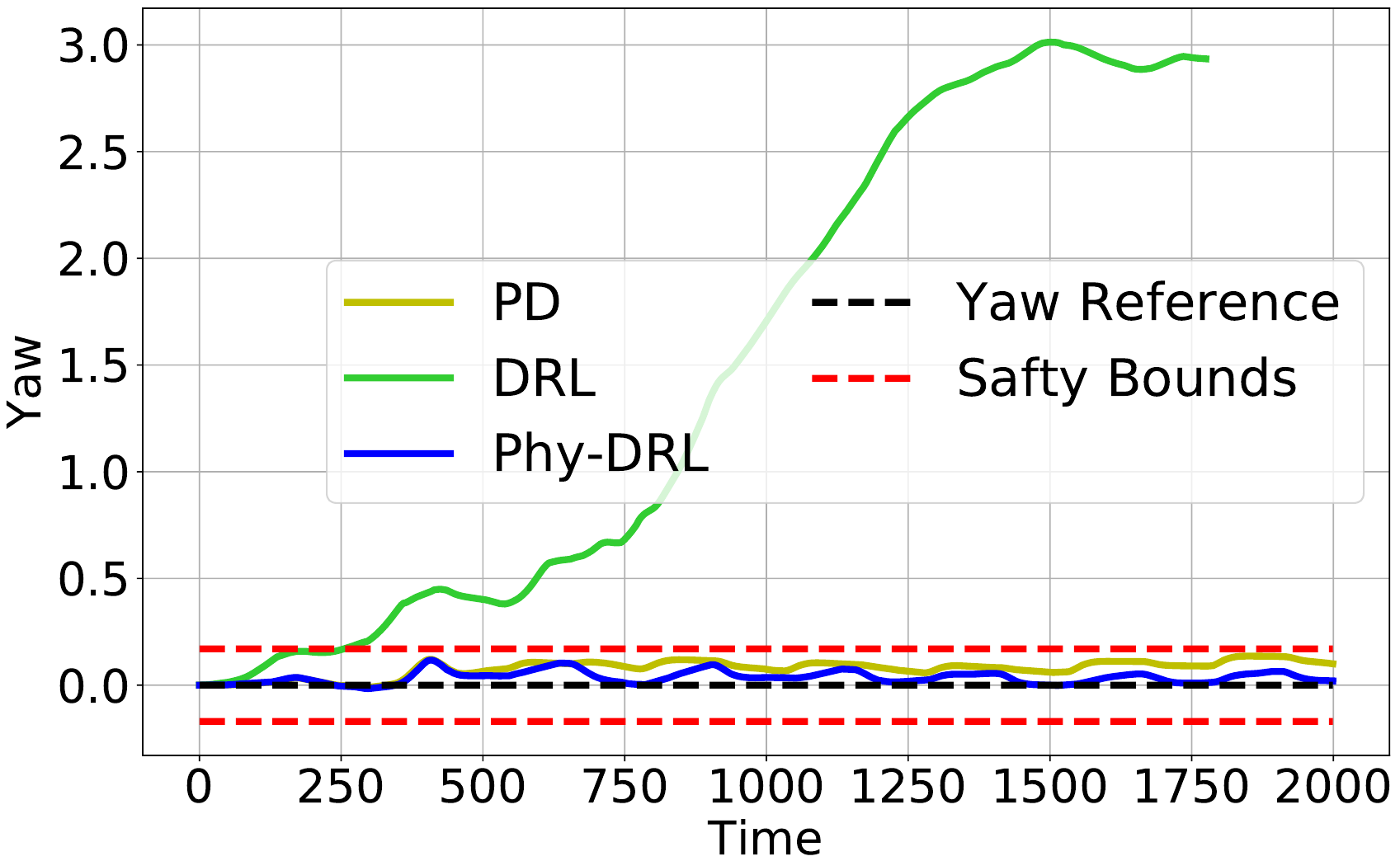}} 
    \centering
    \subfloat[wet road: $r_{\text{x}}=  -1$ m/s]{\includegraphics[width=0.245\textwidth]{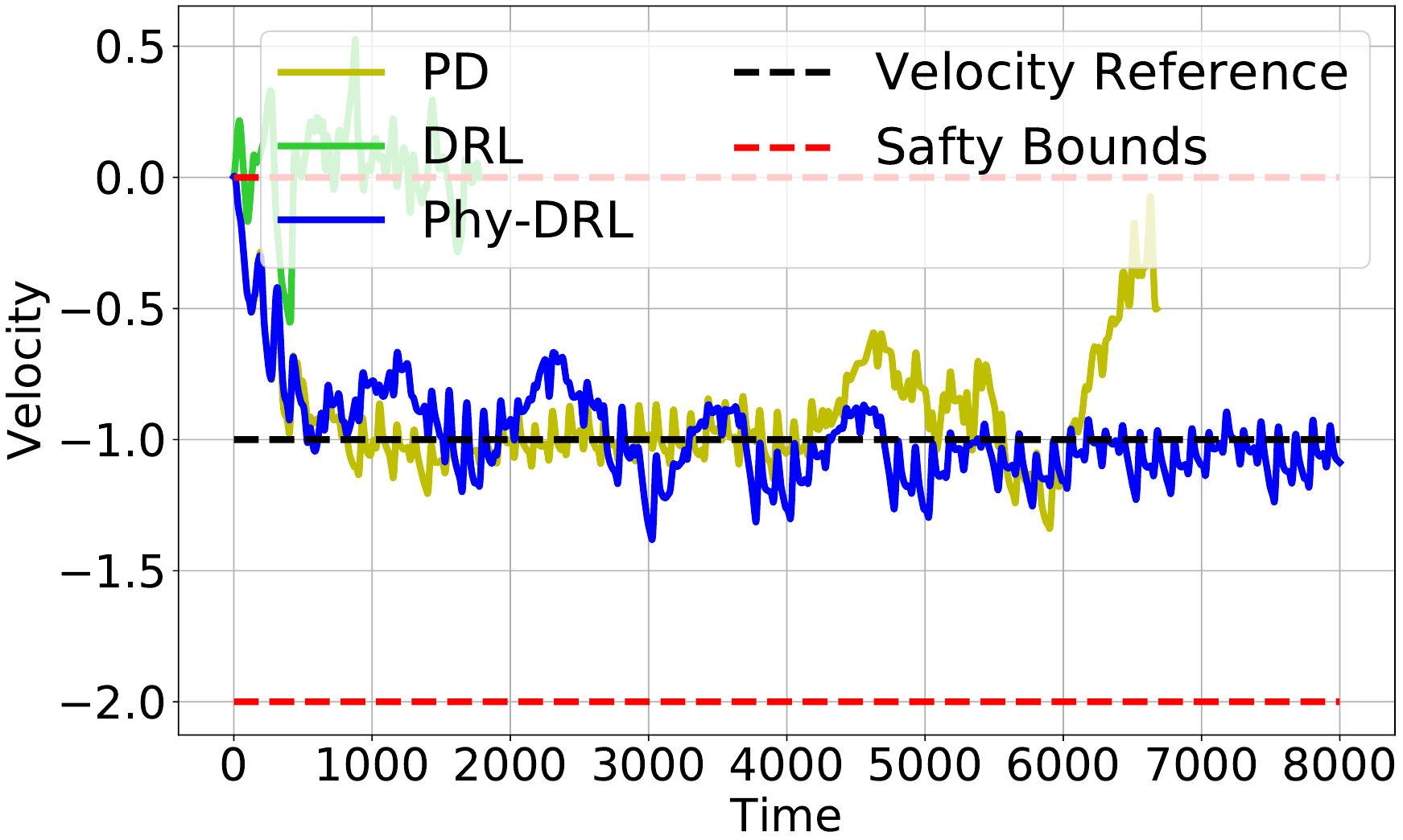}}
\caption{Yaw and velocity trajectories under velocity commands $r_{\text{x}} = -1 ~\textit{or}~ 1$ m/s on snow and wet road.}
\label{rfret}
\end{figure*}

\begin{figure*}[h]
    \centering
    \subfloat[Snow Road: $r_{\text{x}}=  0.5$ m/s]{\includegraphics[width=0.245\textwidth]{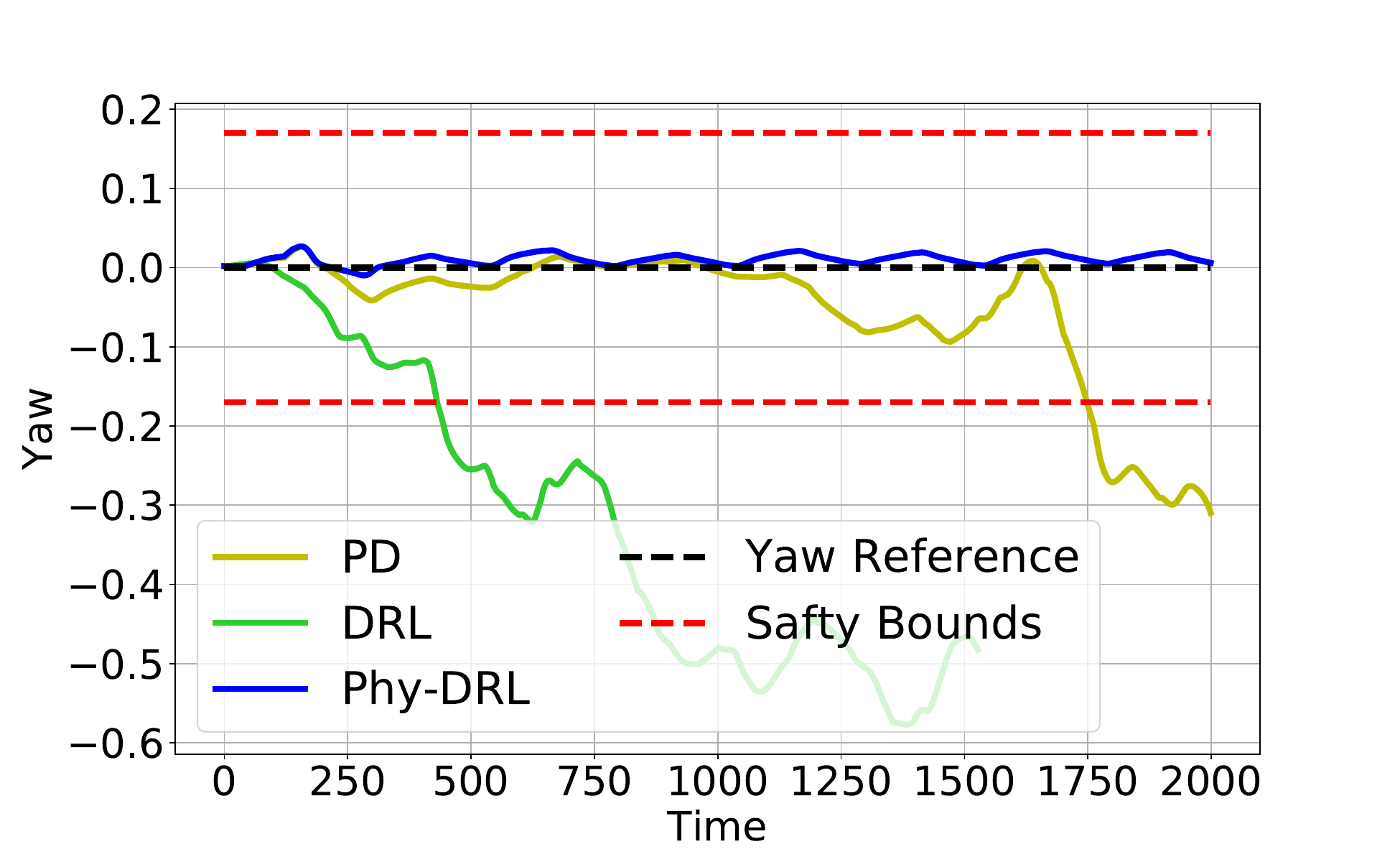}} 
    \centering
    \subfloat[Snow Road: $r_{\text{x}}=  0.5$ m/s]{\includegraphics[width=0.245\textwidth]{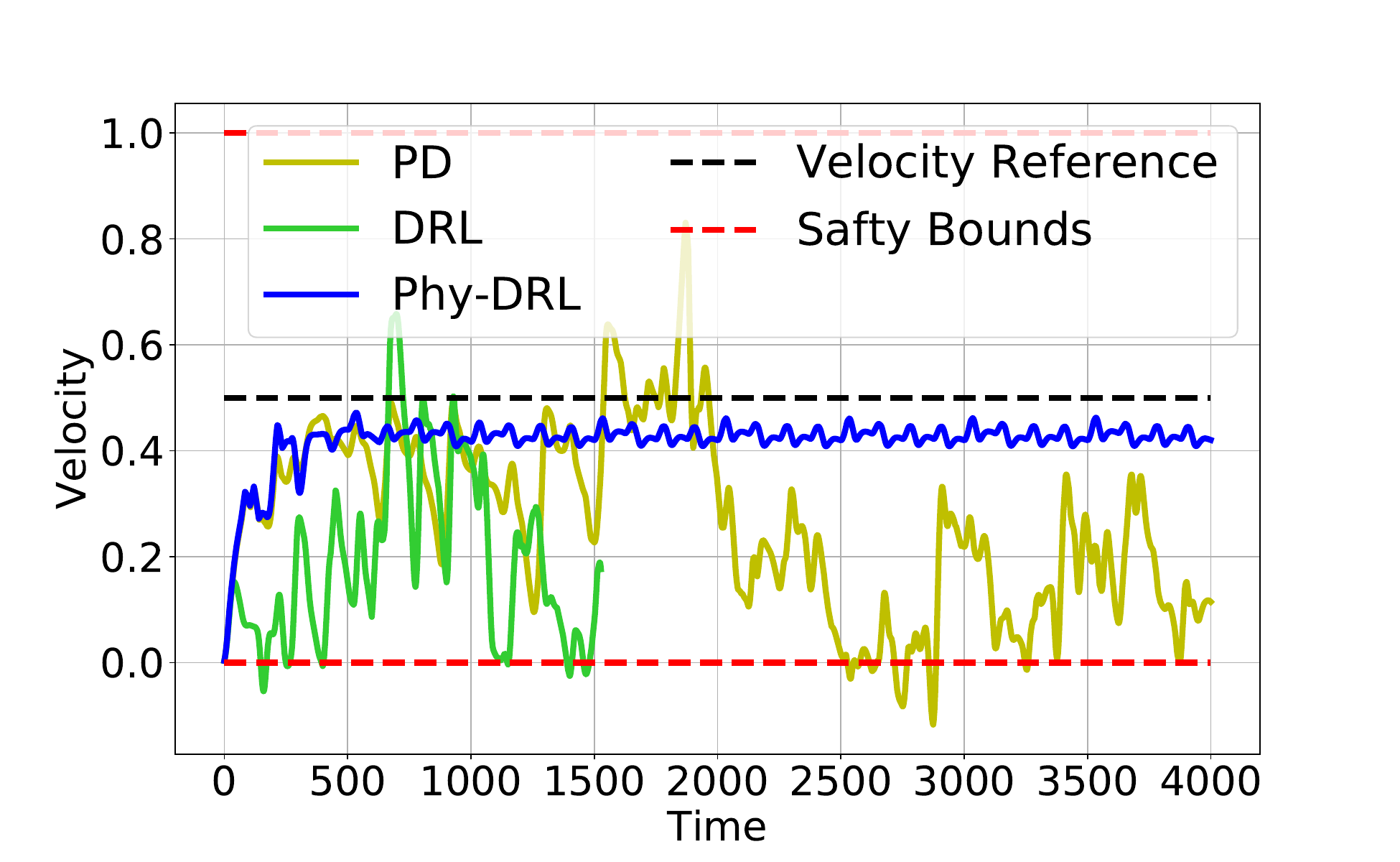}} 
        \centering
        \subfloat[Wet Road: $r_{\text{x}}=  -0.4$ m/s]{\includegraphics[width=0.245\textwidth]{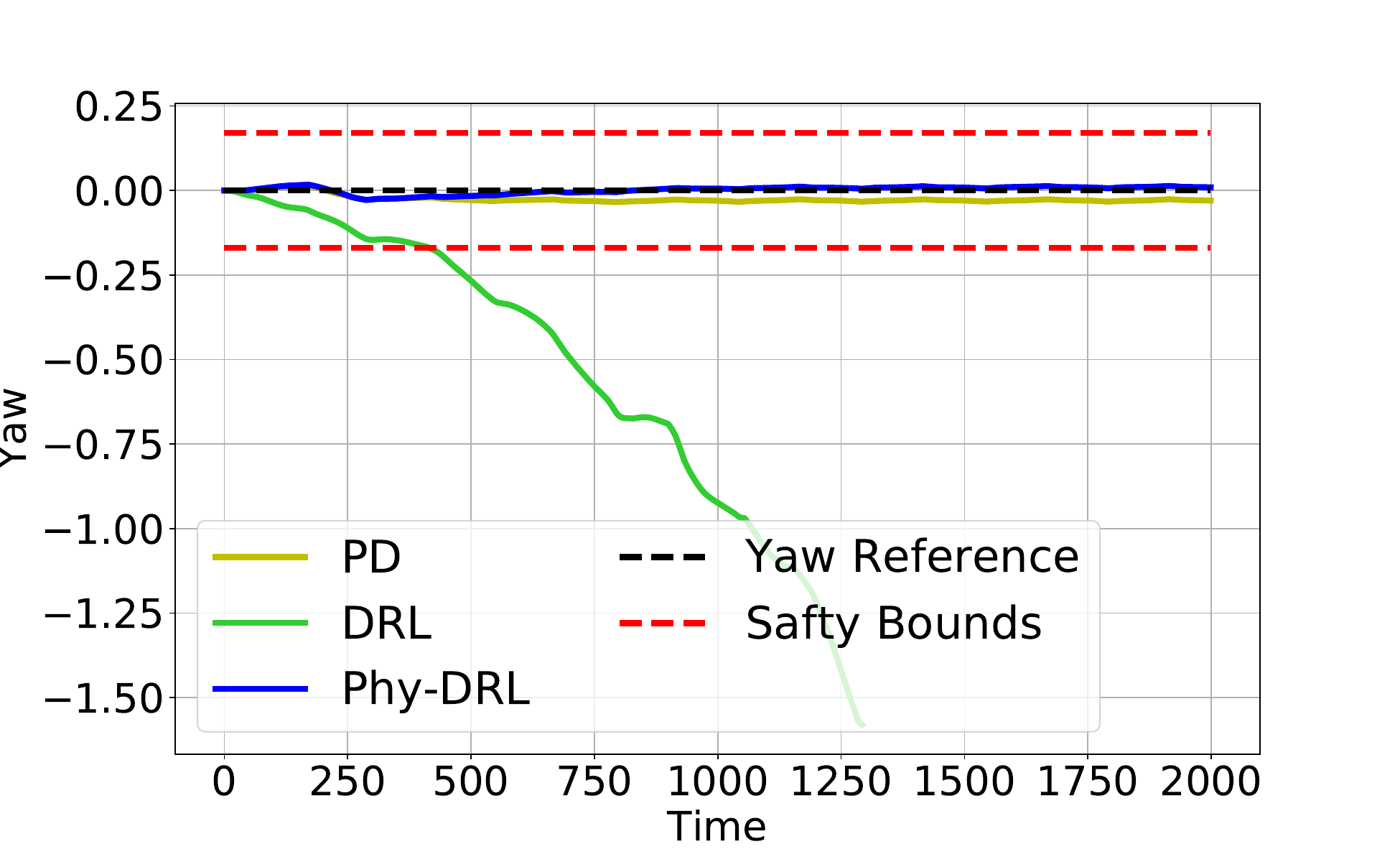}} 
    \centering
    \subfloat[Wet Road: $r_{\text{x}}=  -0.4$ m/s]{\includegraphics[width=0.245\textwidth]{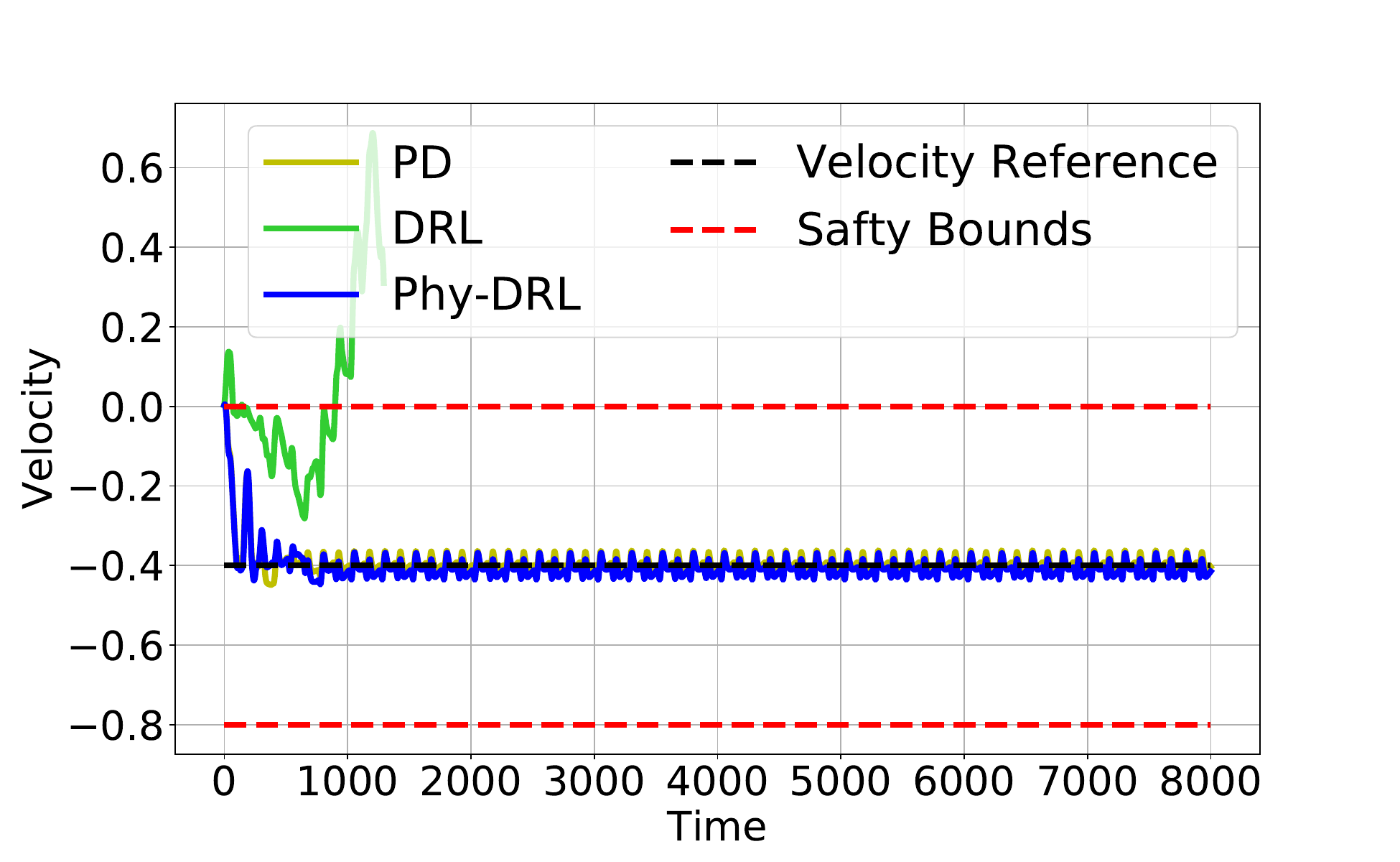}} 
\caption{Yaw and velocity trajectories under velocity commands $r_{\text{x}} = 0.5 ~\textit{or}~ -0.4$ m/s on snow and wet road.}
\label{rfretcc}
\end{figure*}

\section{Experiments}
\label{exp}
In this section, we evaluate the effectiveness of the proposed training algorithm using worst-case conditions and uniform sampled conditions commonly used in literature \cite{muratore2022robot, liu2022goal, kirk2023survey}. We evaluate these two condition sampling strategy on simulated cart-pole system, a 2D quadrotor, and a quadruped robot, and we cite the following safety samples from \cite{Phydrl1} for defining safety metrics. 
\begin{align}
&\textbf{Internal-Envelope (IE) sample $\widetilde{\mathbf{s}}$:} \\ \nonumber
& ~\text{if}~\mathbf{s}(1) =  \widetilde{\mathbf{s}} \in {\Omega}, ~\text{then}~\mathbf{s}(k) \in {\Omega}, \forall k \in \mathbb{N}. \label{ies} \\
&\textbf{External-Envelope (EE) sample $\widetilde{\mathbf{s}}$:} \\ \nonumber
&~\text{if}~\mathbf{s}(1) =  \widetilde{\mathbf{s}} \in {\mathbb{X}}, ~\text{then}~\mathbf{s}(k) \in \mathbb{X} \setminus \Omega, \exists  k \in \mathbb{N}.\label{ees}
\end{align}

Intuitively, \textbf{IE} samples means that the system starts from the safety envelope and it always stay in the safety envelope. \textbf{EE} means that the system starts from the safety set but not in the safety envelope and always stays in the safety set. 

\subsection{DRL Policy Setup}
We implement our policy using Phy-DRL framework \cite{Phydrl1}, where the $a_{drl}$ is implemented based on DDPG algorithm \cite{lillicrap2015continuous}. The action value function and actor network are both parameterized using a multi-layer-perceptions (MLP) model. The reward function is designed as the format of \cref{reward}, where the matrix $\mathbf{P} $ and $\mathbf{H}$ are obtained as in \cite{Phydrl2} by solving LMI problems for each robot using their own linear dynamics models.

% \begin{table}
%  \footnotesize{\caption{Training episodes and failed episodes.}\label{table:failureyu}
% \begin{tabular}{c|c|c|c}
% \toprule[1.5pt]
% \rule{0pt}{1.1EM}%
% Sampling IDs & EPs Num. & Failed Eps & Failure rate \\
% \midrule
% \rule{0pt}{1.1EM}%
% worst-case & 170 & 1 & 0.59\%\\
% random  & 170 & 154 & 90.6\%\\
% worst-case-w.t.& 170 & 0 & 0.\%\\
% random-w.t. & 170 & 85 & 50.00\%\\
% \bottomrule[1.5pt]
% \end{tabular}}
% \end{table} 

\begin{table}[http]
\vspace{5pt}
\caption{Training episodes and failed episodes for different period of training and numbers of sampled conditions. The failure episodes means the episode in which the system violates the safety constraints.}
\aboverulesep=0ex % Solution part 1 of 3
\belowrulesep=0ex % Solution part 1 of 3
\center
\begin{tabular}{c|c|c|c}
\toprule[1.5pt]
\rule{0pt}{1.1EM}%
\textbf{Sampling IDs} & \textbf{EPs Num} & \textbf{Failed Eps} & \textbf{Failure rate} \\
\midrule
worst-case (2-3) & 30 & 3 & 10.0\%\\
random (2-3) & 30 & 29 & 96.67\%\\
worst-case-w.t.(2-3)& 30 & 9 & 29.0\%\\
random-w.t. (2-3) & 30 & 15 & 50.00\%\\
\midrule
worst-case (2-4) & 80 & 3 & 3.75\%\\
random (2-4) & 80 & 74 & 92.5\%\\
worst-case-w.t.(2-4)& 80 & 0 & 0.\%\\
random-w.t. (2-4) & 80 & 46 & 57.50\%\\
\midrule
worst-case (2-5) & 170 & 1 & 0.59\%\\
random (2-5)  & 170 & 154 & 90.6\%\\
worst-case-w.t. (2-5) & 170 & 0 & 0.\%\\
random-w.t.(2-5) & 170 & 85 & 50.00\%\\
\bottomrule[1.5pt]
\end{tabular}
\label{table:failure}
\end{table}

\subsection{Cart pole}
\label{carpole}
In the cart pole case study, the objective is to learn a safe policy that stabilizes the pole from as many initial conditions as possible without violating safety constraints. For \cref{ALG1}, we let $p = 2$ and $q_{1} = q_{2} = q_{3} = 5$, which leads to in total 170 episodes, calculated using \eqref{numeps}. For the uniform sampling scheme, we let the initial position, velocity, angle, and angular velocity be uniformly sampled over the intervals $[-0.9, 0.9]$, $[-3.0, 3.0]$, $[-0.8, 0.8]$, and $[-4.5, 4.5]$, respectively. The bounds of intervals are the same as those of the safety envelope used for worst-case conditions generations. For training, the maximum length of one episode is 500 steps. Additionally, we introduce a terminal condition to the training episode that stops the system from running when a violation of safety occurs in training: $\widehat{\gamma} = 1$ if $|x(k)| \ge 0.9 $  or  $|\theta(k)| \ge 0.8$, and 0, otherwise. Summarily, we reset episodes if the maximum step of system running is reached, or $\widehat{\gamma}=1$.

We consider four sampling methods for Phy-DRL training. They are named 1) `worst-case' (i.e., \cref{ALG1} with termination condition),  2) `random' (i.e., random sampling with termination condition), 3) `worst-case-w.t.' (i.e., \cref{ALG1} without termination condition), and 4) `random-w.t.' (i.e., random sampling without termination condition).

We use the areas of IE sample \eqref{ies} and EE sample \eqref{ees} as the safety metrics. During testing, we intentionally introduce random friction force to the system to increase the testing variety, therefore to showcase the robustness of the learned policy. The sample areas of policies trained by the worst-case sampling and random sampling schemes with the episode termination condition are shown in \cref{safesample_x_theta_w}. The sample areas of policies trained without the episode termination condition are shown in \cref{safesample_x_theta_wo}. For each training episode, if the system violates the safety constraint \eqref{aset2}, this episode is marked as a failed episode. \cref{table:failure} summarizes the number of failed episodes and the failure rate of the four sampling schemes. Observing \cref{safesample_x_theta_w}, \cref{safesample_x_theta_wo}, and \cref{table:failure}, we discover:  
\begin{itemize}
\item The worst-case sampling procedure (i.e., \cref{ALG1}) empowers the Phy-DRL with fast and efficient training towards a safety guarantee for both \textbf{with} and \textbf{without} the episode termination settings, i.e., 
    the learned policy successfully rendering the safety envelope invariant (see \cref{safesample_x_theta_w} (a) and (c), and \cref{safesample_x_theta_wo} (a) and (c)).
    \item Compared with random sampling, the worst-case sampling has much smaller failure rate of episodes (see last two columns in \cref{table:failure}). The root reason goes back to the solutions of worst-case samples in \cref{solutionsboundary}, which automatically avoids many samples that are physically infeasible to control. Meanwhile, the safety areas of action policy trained using random sampling are much smaller (see \cref{safesample_x_theta_w} (b) and (d), and \cref{safesample_x_theta_wo} (b) and (d)).
\end{itemize}

Moreover, we are interested in how many worst-cases and training period are needed for learning a robust safe policy. Therefore, we consider two additional settings: (2-3) denoting $p = 2$ and $q_{1} = q_{2} = q_{3} = 3$ and (2-4) denoting $p = 2$ and $q_{1} = q_{2} = q_{3} = 4$. All other settings are identical to the experiments in \cref{carpole}. For those two settings, the areas of IE and EE samples of action policies trained under the worst-case sampling and random sampling schemes are shown in \cref{addexx1exp1} and \cref{addexx1exp2}, respectively. Meanwhile, \cref{table:failure} summarizes the number of failed episodes and the failure rate of the four sampling schemes. It can be seen from \cref{table:failure}, the number of failed episodes for worst-case training is consistently less than the random condition training. As demonstrated in \cref{addexx1exp1}, \cref{addexx1exp2}, by training using few conditions, the Phy-DRL can already learn a safe policy. Another reason is that the build-in model based policy in Phy-DRL can efficiently guide the exploration of the policy. We also note from \cref{table:failure}, the number of failed episodes for worst-case condition is very small but not very consistent. We attribute this to the randomness of the exploratory action during policy learning.

\subsection{2D Quadrotor} \label{quadffer}
In this case study, we use the 2D quadrotor simulator provided in Safe-Control-Gym \cite{yuan2022safe} as an experimental system. It is characterized by $(x, z)$ -- the translation position of the CoM of the quadrotor in the $xz$-plane, $\theta$ -- the pitch angle, and their velocities $v_x = \dot{x}$, $v_z = \dot{z}$, and $v_{\theta} = \dot{\theta}$. The mission of the action policy is to stabilize the quadrotor at the waypoint ($r_{x}$, $r_{z}$, $r_{\theta}$) under safety constraints: 

\begin{align}
\!\!\!\!\left| x\!-\!r_{x} \right| \!\le\! 0.5~\text{m}, ~\left| {z \!-\! r_{\text{z}}} \right| \!\le\! 0.8~\text{m}, ~ \left| {\theta \!-\! r_{\theta}} \right| \!\le\! 0.8~\text{rad}.
\end{align}

In this experiment, we set the training episode as 400 with the same episode length as 500. For comparison, we train three policies: $\text{Phy-DRL}_{\text{wc}}$, $\text{Phy-DRL}_{\text{ran}}$, and $\text{DRL}_{\text{CLF-wc}}$, denoting trained Phy-DRL polices using worst-case sampling, random sampling, and pure data-driven DRL policy trained using worst-case sampling and CLF-like reward (proposed in \cite{westenbroek2022lyapunov}), respectively. For the random sampling, we let the initial $x, z, \theta, v_x, v_z, v_{\theta}$ be uniformly sampled over the intervals $[1.5, 2.5]$, $[3.2, 4.8]$, $[-0.8, 0.8]$, $[-1, 1]$, $[-10, 10]$, and $[-45, 45]$, respectively. The intervals' bounds are the same as those of the safety envelope used for worst-case sampling. 

For testing, we set initial velocities as zeros. The IE samples \eqref{ies} of considered three policies are shown in \cref{quaapcv} (a)-(c). It can be seen from \cref{quaapcv}, that Phy-DRL -- powered by worst-case sampling -- features much more stable and fast training towards safety guarantee. Meanwhile, the reward's training curves (five random seeds) of $\text{Phy-DRL}_{\text{wc}}$ v.s. $\text{Phy-DRL}_{\text{ran}}$ are shown in \cref{quaapcv} (d), where the training from the random conditions cause the large variance for policy learning. This could lead the policy to be sub-optimal, as also observed in \cite{mehta2020active}.

At the last, $\text{DRL}_{\text{CLF-wc}}$ has zero IE sample given the current training episodes. This suggests equipping DRL with only a worst-case sampling scheme cannot efficiently search for a safe action policy, and it needs a guide policy as in Phy-DRL \cite{Phydrl2} or more training steps.

\subsection{Quadruped Robot} \label{ccrobnt}
\cref{carpole} and \cref{quadffer} have focused on safety performance; thus, this section focuses on robustness evaluation of safe policy in quadruped locomotion. The mission of the action policies includes safe lane tracking along the x-axis and safe velocity regulation. We consider the following safety constraints:
\begin{align}
&\left| {\text{yaw}} \right| \le 0.17 ~\text{rad}, ~\left| {\text{CoM x-velocity} - r_{\text{x}}} \right| \le |r_{\text{x}}|, \\ \nonumber
& ~\left| {\text{CoM z-height} \!-\! 0.24 \!~\text{m}} \right| \le 0.13 \!~\text{m},  \label{safetysetrobot}
\end{align}
where $r_{\text{x}}$ denotes velocity reference or command. The model-based design follows \cite{Phydrl1}. 
We train and evaluate the policy in Pybullet simulator with varying road conditions.
For the worst-case training, we consider only one worst-case sample: forward velocity command $r_{\text{x}} = 1$ m/s and snow road (low friction on terrain). We perform comparisons with the other two models:  proportional-derivative (PD) controller developed in \cite{da2021learning} and pure data-driven DRL policy \cite{lillicrap2015continuous}. The Phy-DRL with the worst-case sampling is trained for only $10^{6}$ steps, while the pure data-driven DRL is trained for $10^{7}$ steps. 

We also consider a testing environment very different from a training environment:  backward velocity command $r_{\text{x}} = -1$ m/s and wet road. The yaw and CoM-velocity trajectories are shown in \cref{rfret}. In addition, another groups of experimental results for forward velocity command $r_{\text{x}}=  0.5$ m/s and snow road, and backward velocity command $r_{\text{x}}=  -0.4$ m/s and wet road are shown in \cref{rfretcc}.  From \cref{rfret}, \cref{rfretcc}, we can see that  compared with DRL and PD, the Phy-DRL's action policy (trained using only one worst-case sample) has much higher tracking performance under different road conditions while strictly following the safety regulations in \cref{safetysetrobot}. 

After training in the simulation, we also transferred the learned policy to the real robot; see a demonstration video available at \href{https://www.dropbox.com/scl/fi/4j6357blxnos4x1s7uj3p/final.mp4?rlkey=utvnrv23tgb64x3fh499yrij8&dl=0}{\color{blue} anonymous link}. We note that crossing the white lane in the video means violating the first safety regulation (i.e. $\left| {\text{yaw}} \right| \le 0.17 ~\text{rad}$) in \cref{safetysetrobot}. We also note that the video does not include pure data-driven DRL policy as we found that DRL policy can not make progress after training $10^{7}$ steps from scratch.

\section{Conclusion and Discussion} \label{conljus}
This paper proposes the sparse worst-case sampling for Phy-DRL training. The particular design aims of worst-case sampling include i) automatically avoiding state samples that are physically infeasible and ii) focusing the training on corner cases represented by worst-case samples. The spare worst-case sampling makes the Phy-DRL features much more efficient and fast training towards safety guarantee. 

Under worst-case sampling, a potential negative issue could be increased instability of the model-based policy in Phy-DRL, when training starts from boundary. To address this, Phy-DRL shall run on a fault-tolerant software architecture called Simplex \cite{sha2001using}. In Simplex, we use the Phy-DRL as the complex and high-performance controller, which may have unknown defects. Meanwhile, Simplex's high-assurance controller (HAC) is function-reduced and simplified but verified, and it only guarantees the system's basic stable and safe operations. HAC is thus complementary to Phy-DRL and coordinated by a monitor. For example, the monitor triggers the switch from Phy-DRL to HAC once the real-time system states (under the control of Phy-DRL) leave the safety envelope. In other words, the HAC takes over at the cost of lower performance. When the system returns to the safety set, the Phy-DRL can be restarted and control can be retaken.

\bibliographystyle{ACM-Reference-Format}
\bibliography{ref}

% \n
\clearpage
\appendix
\section{Appendix}
\subsection{Auxiliary Lemmas} \label{Aux}
\begin{lemma} [Positive Definiteness \cite{bhatia2009positive}]
A matrix $\mathbf{A} \in \mathbf{R}^{n \times n}$ is called positive definite if it is symmetric and all its eigenvalues are positive. In other words, there exists an orthogonal matrix $\mathbf{Q} \in \mathbf{R}^{n \times n}$, such that 
\begin{align}
\mathbf{Q}^{\top} \cdot \mathbf{A} \cdot \mathbf{Q} = \left[ {\begin{array}{*{20}{c}}
{{\lambda _1}}&0& \cdots &0\\
0&{{\lambda _2}}& \cdots &0\\
 \vdots & \vdots & \vdots &0\\
0&0& \cdots &{{\lambda _n}}
\end{array}} \right], \\
~~\text{with}~\lambda_1 > 0, ~\lambda_2 > 0, ~\ldots, ~\lambda_n > 0.  \nonumber
\end{align}
\label{positive}
\end{lemma}

\subsection{Proof of \cref{solutionsboundary}} \label{pfworst}
    The $\mathbf{P} \succ 0$ means the matrix $\mathbf{P}$ is positive definite. In light of \cref{positive} in \cref{Aux}, there exists an orthogonal matrix $\mathbf{Q}(\mathbf{P}) \in \mathbb{R}^{n \times n}$, such that 
\begin{align}
&\mathbf{Q}^{\top}(\mathbf{P}) \cdot \mathbf{P} \cdot \mathbf{Q}(\mathbf{P}) \nonumber\\
&= \underbrace{\left[ {\begin{array}{*{20}{c}}
{{\lambda _1}(\mathbf{P})}&0& \cdots &0\\
0&{{\lambda _2}(\mathbf{P})}& \cdots &0\\
 \vdots & \vdots & \vdots &0\\
0&0& \cdots &{{\lambda _n}(\mathbf{P})}
\end{array}} \right]}_{\triangleq ~\Lambda(\mathbf{P})},\\
&~~\text{with}~\lambda_1(\mathbf{P}) > 0, ~\lambda_2(\mathbf{P}) > 0, ~\ldots, ~\lambda_n(\mathbf{P}) > 0.  \label{pso1}
\end{align}

Considering \cref{pso1}, the ${\mathbf{s}^\top}\cdot {\mathbf{P}} \cdot \mathbf{s} = \varphi$ equates to 
\begin{align}
{\mathbf{s}^\top}\cdot {\mathbf{P}} \cdot \mathbf{s} = {\mathbf{s}^\top}\cdot \mathbf{Q}(\mathbf{P}) \cdot \Lambda(\mathbf{P}) \cdot \mathbf{Q}^{\top}(\mathbf{P}) \cdot \mathbf{s} = \varphi, \label{pso2}
\end{align}
whose transformation utilizes a well-known property of orthogonal matrix: $\mathbf{Q}^{\top}(\mathbf{P}) = \mathbf{Q}^{-1}(\mathbf{P})$. 

Let us define $\mathbf{y} \triangleq \mathbf{Q}^{\top}(\mathbf{P}) \cdot \mathbf{s}$. Recalling the $\Lambda(\mathbf{P})$ defined in \cref{pso1}, \cref{pso2} equivalently transforms to 
\begin{align}
\sum\limits_{i = 1}^n {\frac{{{\lambda _i}(\mathbf{P})}}{\varphi}}  \cdot \mathbf{y}_i^2 = 1. \label{pso3}
\end{align}

We next use the mathematical induction to prove that \cref{solutionsboundary2} is the solution to the problem in \cref{pso3}. To complete this proof task, we first let $n = 2$. In this case, we observe from \cref{solutionsboundary2} that  
\begin{align}
[{\mathbf{y}}]_{1} = \sqrt{\frac{\varphi}{{{\lambda _1}( \mathbf{P})}}} \cdot \sin ({{\theta_1}}), ~~~[{\mathbf{y}}]_{2} = \sqrt{\frac{\varphi}{{{\lambda_2}( \mathbf{P})}}} \cdot \cos ({{\theta_1}}), \nonumber
\end{align}

which, in conjunction with the well-know triple angle identity $\sin^2(\theta_{1}) + \cos^2(\theta_{1}) = 1$, directly leads to the formula in \cref{pso3} with $n = 2$. We then consider the case of $n = p > 2$. Assuming \cref{solutionsboundary2} is the solution to the problem in \cref{pso3} with $n = p$, we have 
\begin{align}
\sum\limits_{i = 1}^p {\frac{{{\lambda _i}(\mathbf{P})}}{\varphi}}  \cdot \bar{\mathbf{y}}_i^2 = 1, \label{pso4}
\end{align}
where 
\begin{align}
 [\bar{{\mathbf{y}}}]_{i} \triangleq \begin{cases}
		\sqrt{\frac{\varphi}{{{\lambda _1}( \mathbf{P})}}} \cdot \sin ({{\theta_1}}) \cdot \prod\limits_{m = 2}^{p - 1} {\sin } ({{\theta _m}}), &i = 1 \\
        \sqrt{\frac{\varphi}{{{\lambda _i}( \mathbf{P})}}} \cdot \cos ({{\theta_{i-1}}}) \cdot \prod\limits_{m = i}^{p - 1} {\sin } ({{\theta _m}}), &i \ge 2 
	\end{cases}. \label{pso5}
\end{align}

Based on this observation, we need to prove that \cref{solutionsboundary2} is the solution to the problem in \cref{pso3} with $n = p+1$. For the  $n = p+1$ and $\mathbf{y} \in \mathbb{R}^{p+1}$, according to \cref{solutionsboundary2} to have 
\begin{align}
 [{\mathbf{y}}]_{i} \triangleq \begin{cases}
		\sqrt{\frac{\varphi}{{{\lambda _1}( \mathbf{P})}}} \cdot \sin ({{\theta_1}}) \cdot \prod\limits_{m = 2}^{p} {\sin } ({{\theta _m}}), &i = 1 \\
        \sqrt{\frac{\varphi}{{{\lambda _i}( \mathbf{P})}}}\cdot \cos ({{\theta_{i-1}}}) \cdot \prod\limits_{m = i}^{p} {\sin } ({{\theta _m}}), &i \ge 2 
	\end{cases} \nonumber
\end{align}
which is equivalent to 
\begin{align}
 [{\mathbf{y}}]_{i} \triangleq \begin{cases}
		\sqrt{\frac{\varphi}{{{\lambda _1}( \mathbf{P})}}} \cdot \sin ({{\theta_1}}) \cdot \prod\limits_{m = 2}^{p-1} {\sin } ({{\theta _m}}) \cdot {\sin } ({{\theta_{p}}}), &i = 1 \\
        \sqrt{\frac{\varphi}{{{\lambda _i}( \mathbf{P})}}} \cdot \cos ({{\theta_{i-1}}}) \cdot \prod\limits_{m = i}^{p-1} {\sin }({{\theta _m}}) \cdot {\sin } ({{\theta_{p}}}), &2 \le i \le p \\
        \sqrt{\frac{\varphi}{{{\lambda _{p+1}}( \mathbf{P})}}} \cdot {\cos} ({{\theta_{p}}}), &i = p + 1
	\end{cases}. \label{pso6}
\end{align}
Observing \cref{pso6} and \cref{pso4} with its solution in \cref{pso5}, we have
\begin{align}
&\sum\limits_{i = 1}^{p+1} {\frac{{{\lambda _i}(\mathbf{P})}}{\varphi}}  \cdot \mathbf{y}_i^2 =  \sum\limits_{i = 1}^{p} {\frac{{{\lambda _i}(\mathbf{P})}}{\varphi}}  \cdot \mathbf{y}_i^2 + {\frac{{{\lambda_{p+1}}(\mathbf{P})}}{\varphi}}  \cdot \mathbf{y}_{p+1}^2 \\ \nonumber
&= {\sin}^{2} ({{\theta_{p}}}) \cdot \sum\limits_{i = 1}^{p} {\frac{{{\lambda _i}(\mathbf{P})}}{\varphi}}  \cdot \bar{\mathbf{y}}_i^2 + {\frac{{{\lambda_{p+1}}(\mathbf{P})}}{\varphi}}  \cdot \mathbf{y}_{p+1}^2 \nonumber\\
&= {\sin}^{2} ({{\theta_{p}}}) + {\cos}^{2} ({{\theta_{p}}}) = 1, \nonumber
\end{align}
by which we can conclude that \cref{pso6} is the solution to \cref{pso3} with $n = p+1$. Hereto, we can conclude that the formula in \cref{solutionsboundary2} solves the problem in \cref{pso3}. 

We are finally, recalling the definition of an orthogonal matrix, i.e., $\mathbf{Q}(\mathbf{P}) \cdot \mathbf{Q}^{-1}(\mathbf{P}) = \mathbf{Q}(\mathbf{P}) \cdot \mathbf{Q}^{\top}(\mathbf{P}) = \mathbf{I}_{n}$, we can directly obtain \cref{solutionsboundary2} from the notation $\mathbf{y} \triangleq \mathbf{Q}^{\top}(\mathbf{P}) \cdot \mathbf{s}$. 

\subsection{Ablation study}

\label{additional}
\begin{figure*}[h]
    \centering
    \subfloat[worst-case]{\includegraphics[width=0.25\textwidth]{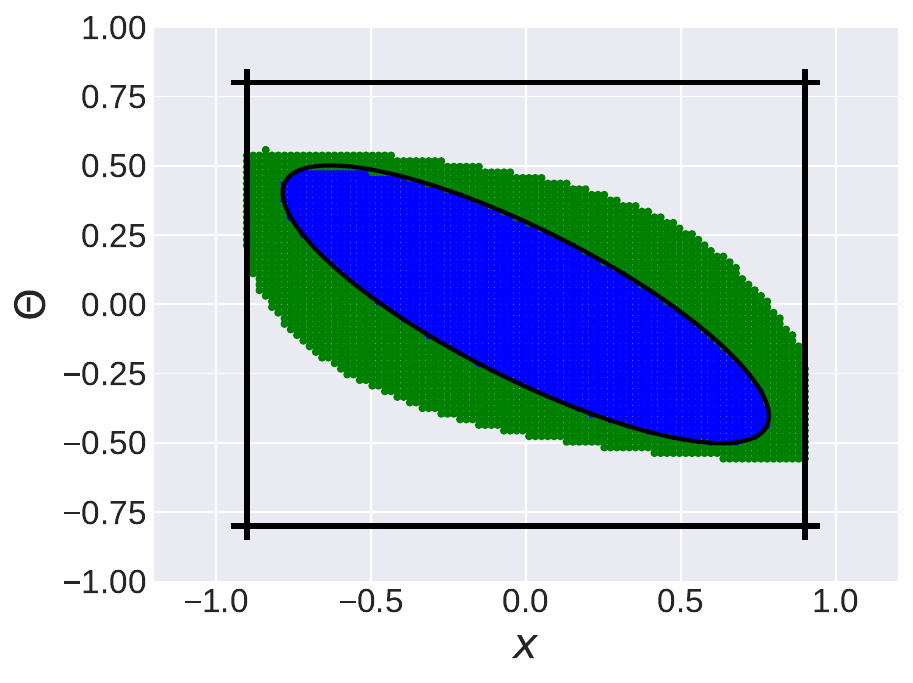}} 
    \centering
    \subfloat[random]{\includegraphics[width=0.25\textwidth]{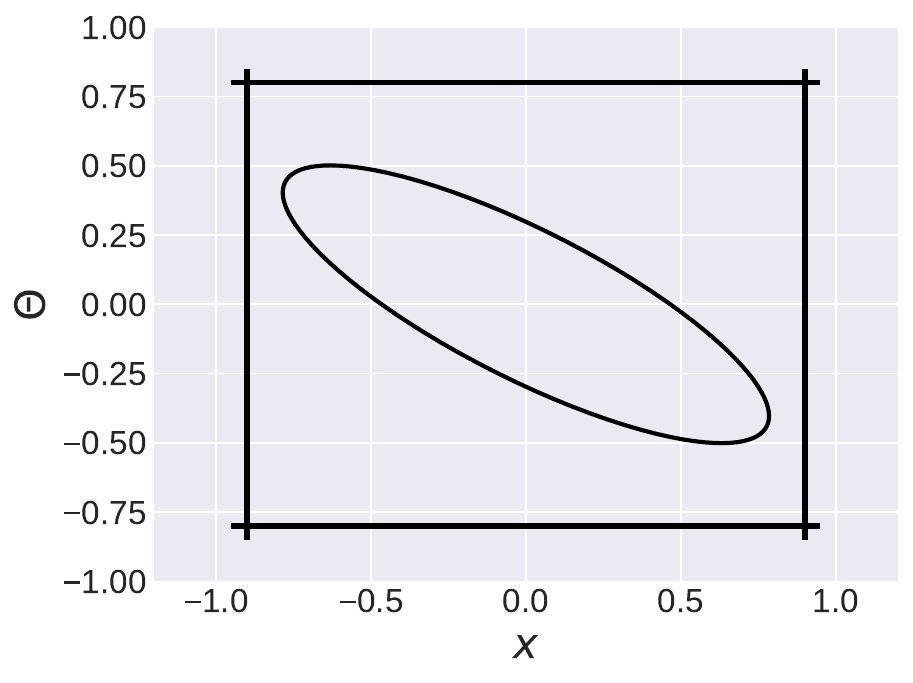}} 
    \centering
    \subfloat[worst-case-w.t]{\includegraphics[width=0.25\textwidth]{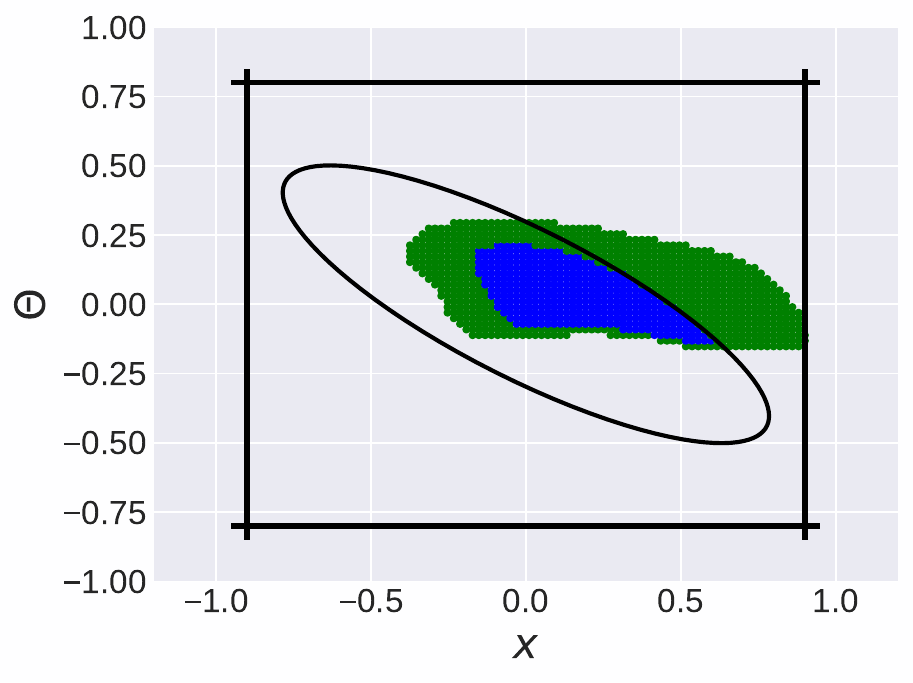}} 
    \centering
    \subfloat[random-w.t]{\includegraphics[width=0.25\textwidth]{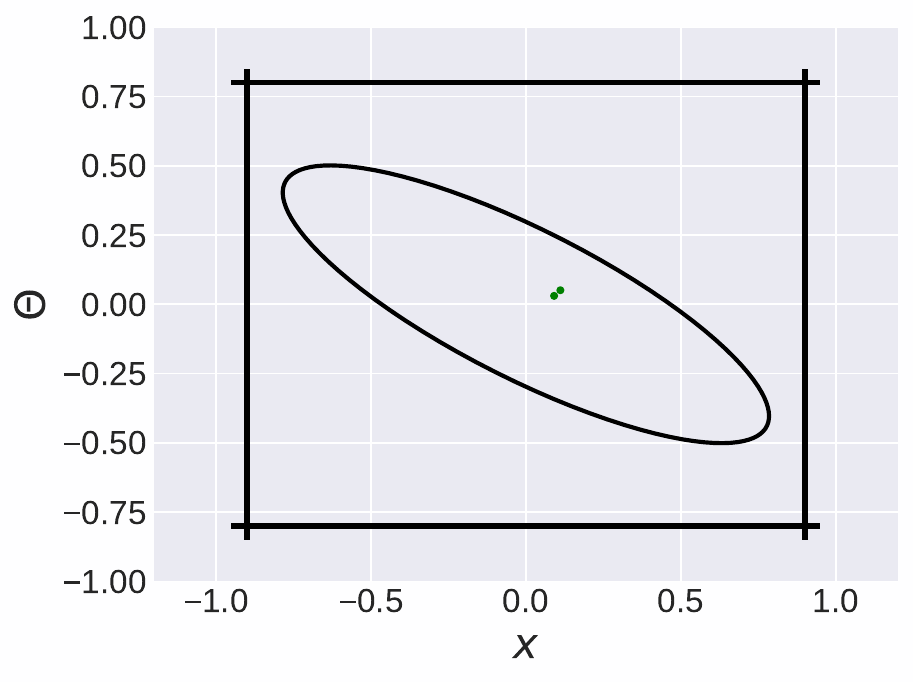}} 
    \centering\\
    \subfloat[worst-case]{\includegraphics[width=0.25\textwidth]{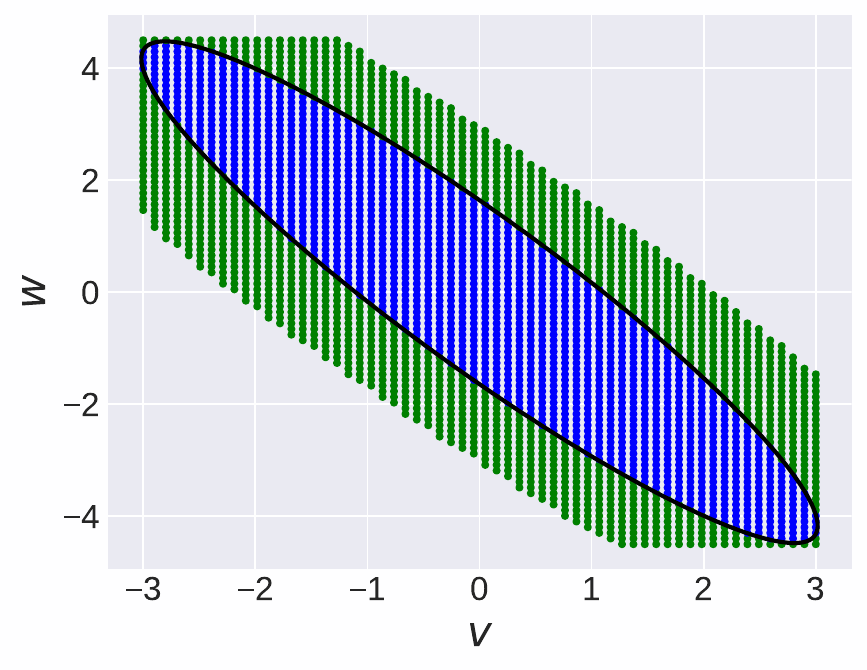}} 
    \centering
    \subfloat[random]{\includegraphics[width=0.25\textwidth]{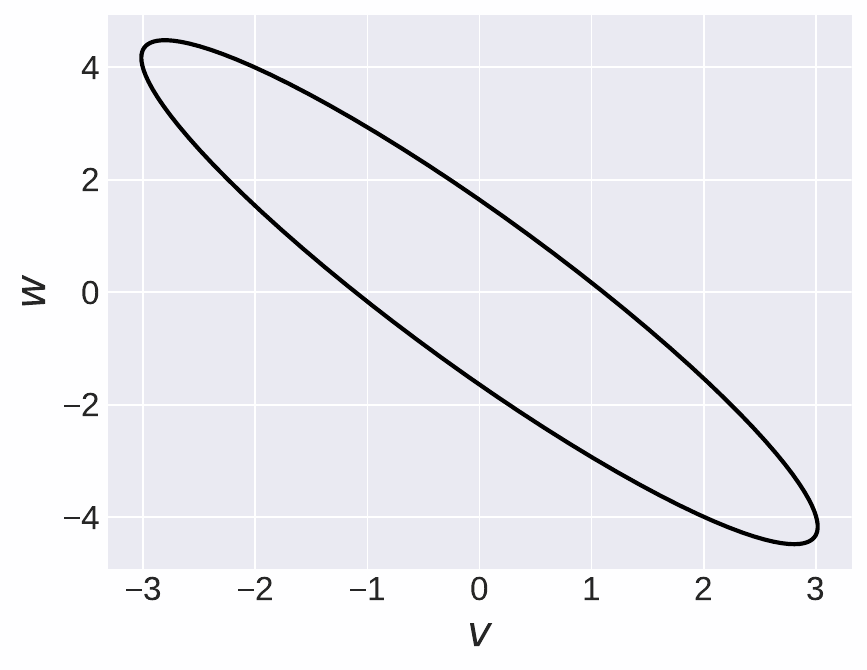}} 
    \centering
    \subfloat[worst-case-w.t]{\includegraphics[width=0.25\textwidth]{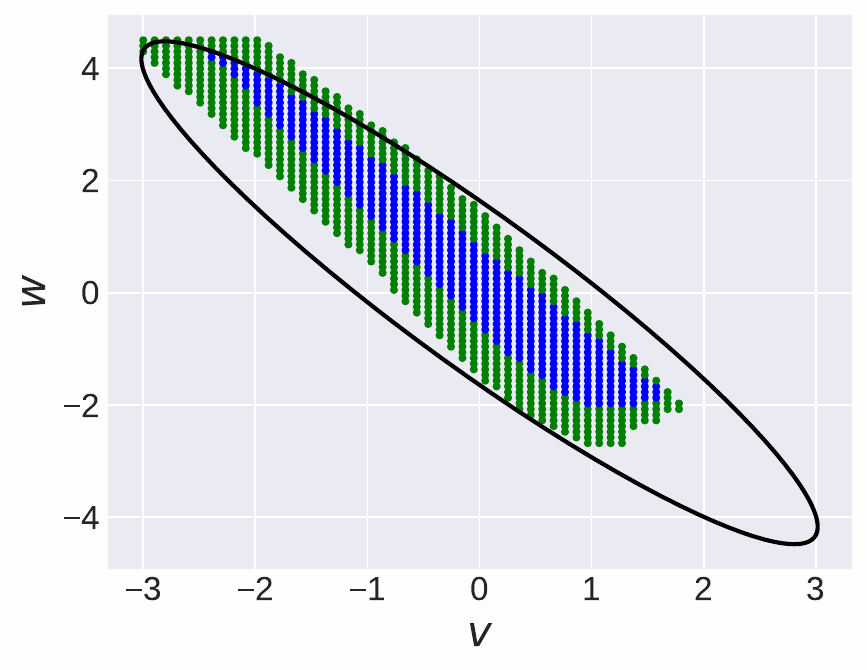}} 
    \centering
    \subfloat[random-w.t]{\includegraphics[width=0.25\textwidth]{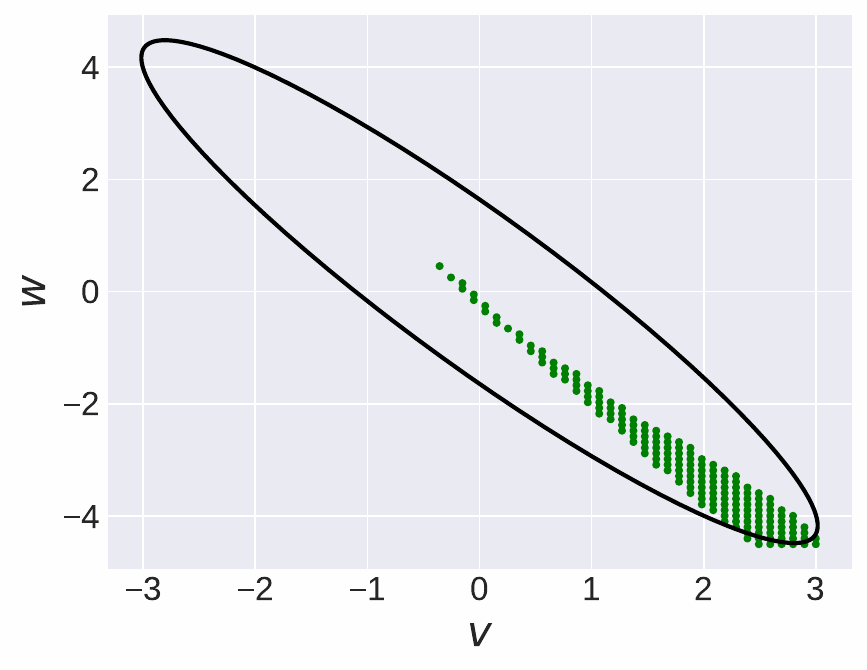}} 
    \centering
    \vspace{-0.0cm}
\caption{\textbf{(2-3)} Worst-case Sampling \textit{v.s.} Random Sampling, with and without termination condition. Blue: area of IE samples. Green: area of EE samples. Ellipse area: safety envelope.}
\label{addexx1exp1}
\end{figure*}

\begin{figure*}[http]
    \centering
    \subfloat[worst-case]{\includegraphics[width=0.25\textwidth]{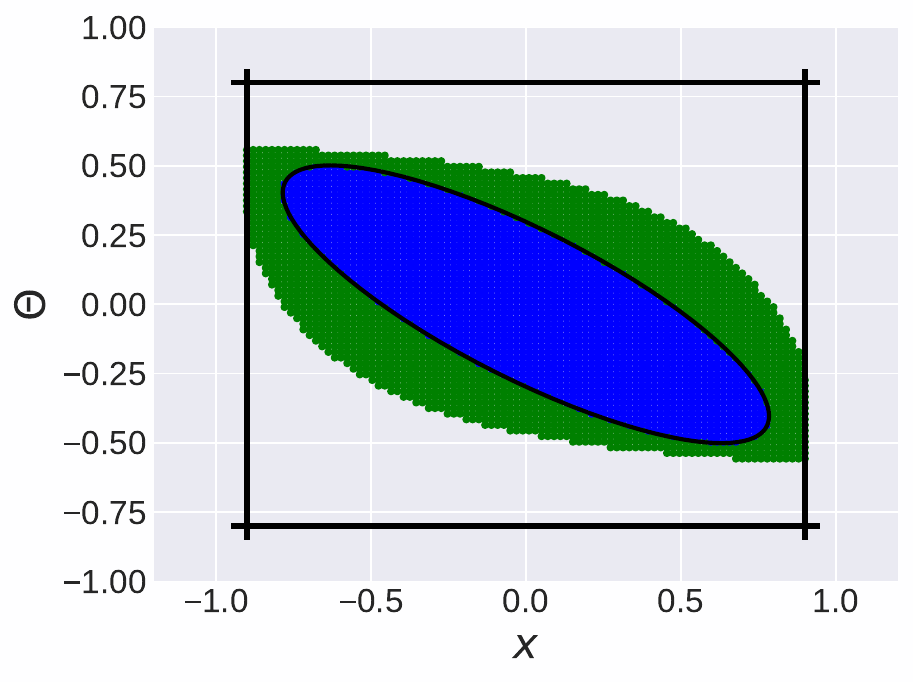}} 
    \centering
    \subfloat[random]{\includegraphics[width=0.25\textwidth]{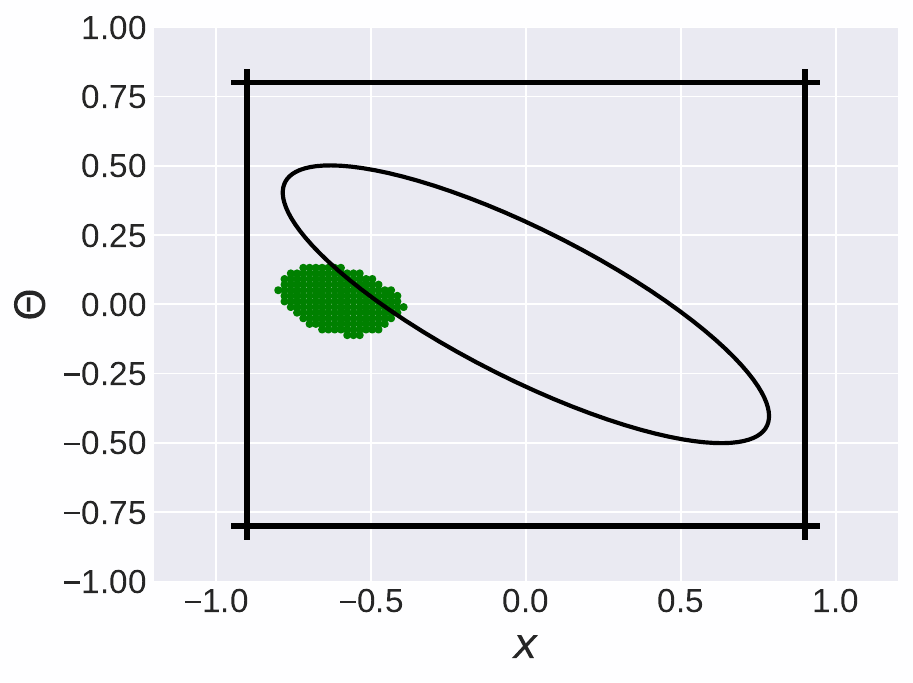}} 
    \centering
    \subfloat[worst-case-w.t]{\includegraphics[width=0.25\textwidth]{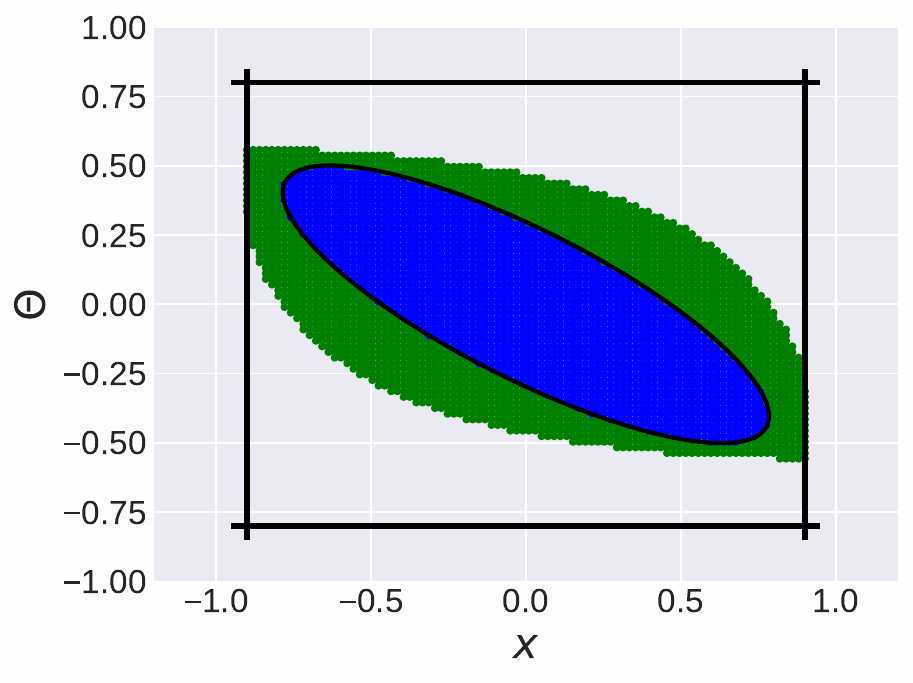}} 
    \centering
    \subfloat[random-w.t]{\includegraphics[width=0.25\textwidth]{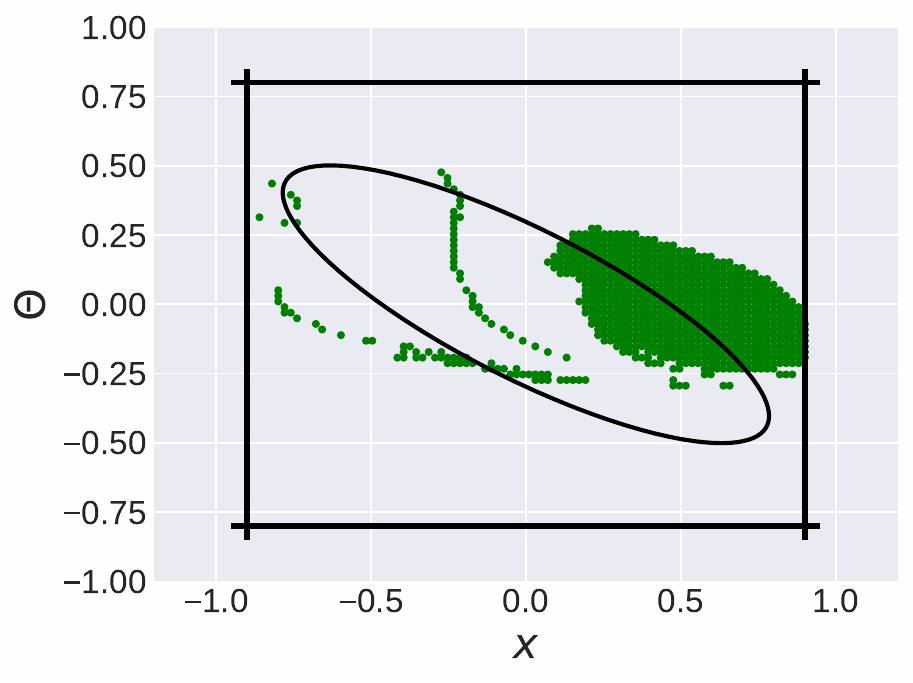}} \\
    \centering
    \subfloat[worst-case]{\includegraphics[width=0.25\textwidth]{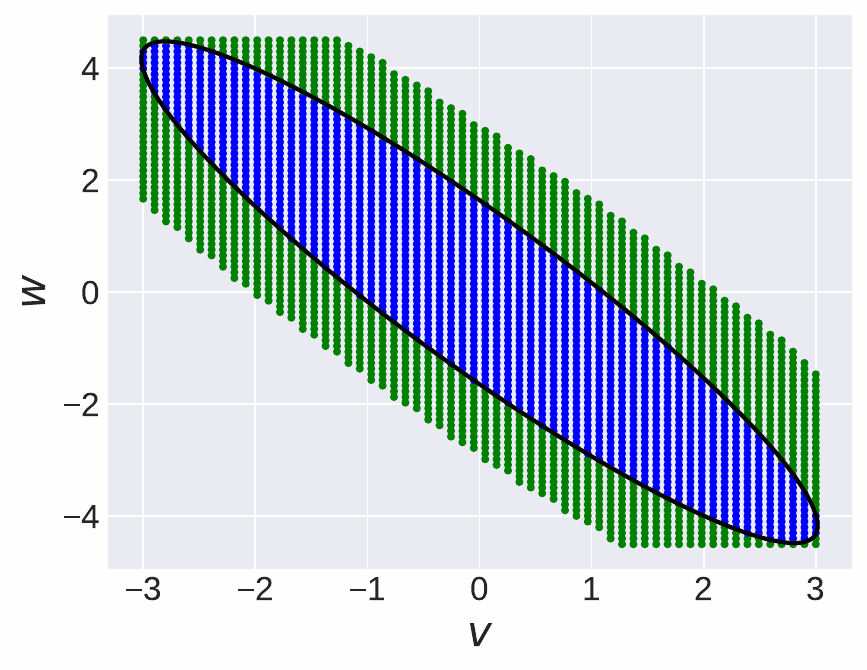}} 
    \centering
    \subfloat[random]{\includegraphics[width=0.25\textwidth]{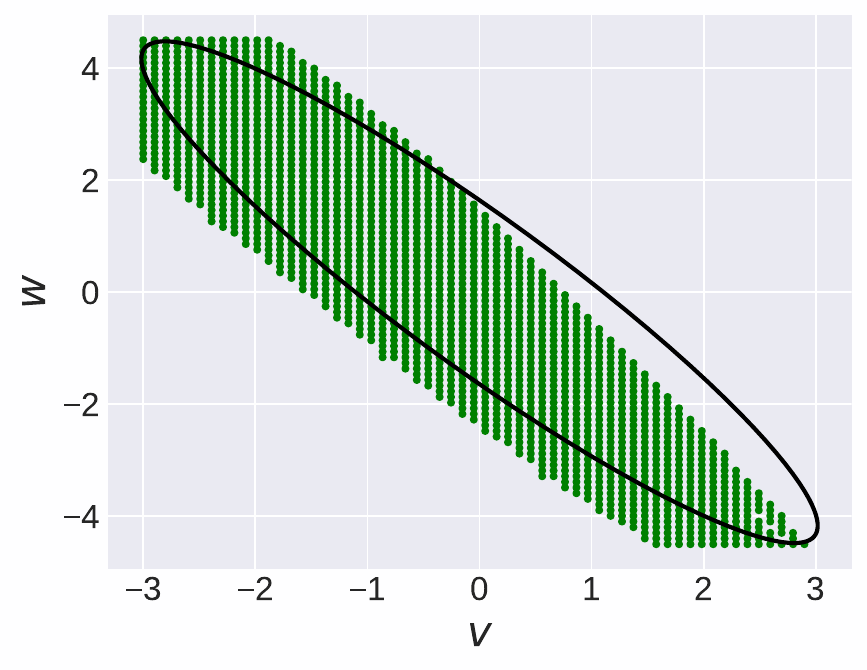}} 
    \centering
    \subfloat[worst-case-w.t]{\includegraphics[width=0.25\textwidth]{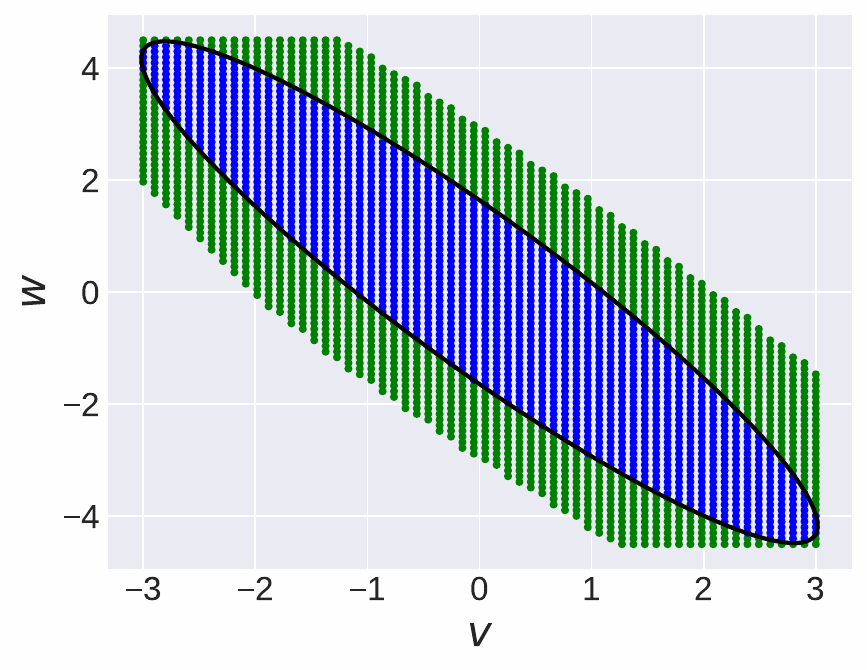}} 
    \centering
    \subfloat[random-w.t]{\includegraphics[width=0.25\textwidth]{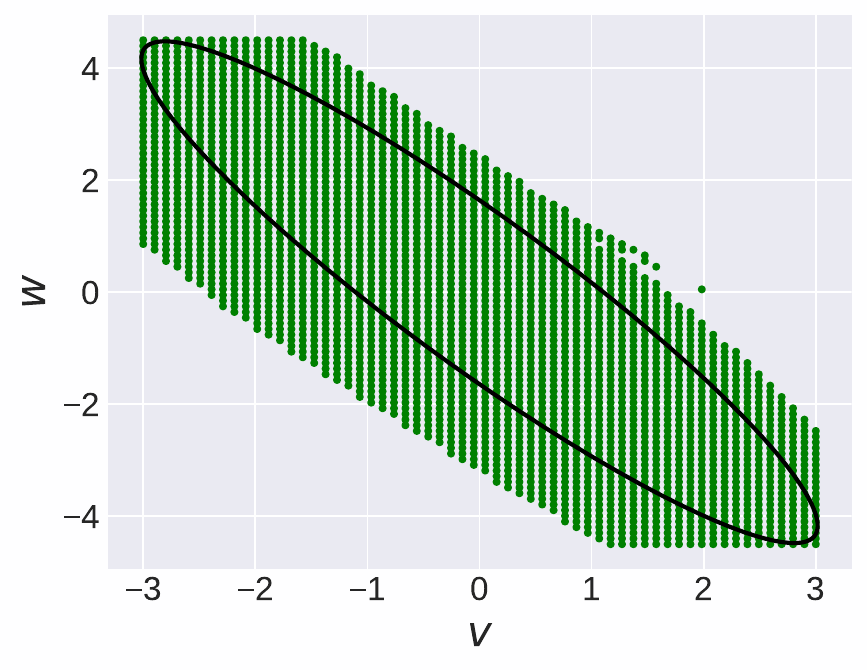}} 
    \centering
\caption{\textbf{(2-4)} Worst-case Sampling \textit{v.s.} Random Sampling, with and without termination condition. Blue: area of IE samples. Green: area of EE samples. Ellipse area: safety envelope.}
\label{addexx1exp2}
\end{figure*}

\end{document}